\documentclass{article}

\PassOptionsToPackage{authoryear}{natbib}

\usepackage[preprint]{neurips_2020}

\usepackage[utf8]{inputenc} 
\usepackage[T1]{fontenc}    
\usepackage{hyperref}       
\usepackage{url}            
\usepackage{booktabs}       
\usepackage{amsfonts}       
\usepackage{nicefrac}       
\usepackage{microtype}      
\usepackage{xparse}
\usepackage{amsthm}
\usepackage{amsmath}
\usepackage{algorithm}
\usepackage{algorithmic}
\usepackage{float}
\usepackage{graphicx}
\usepackage{subcaption}
\usepackage{bbm}

\DeclareFontFamily{U}{mathx}{}
\DeclareFontShape{U}{mathx}{m}{n}{<-> mathx10}{}
\DeclareSymbolFont{mathx}{U}{mathx}{m}{n}
\DeclareMathAccent{\widehat}{0}{mathx}{"70}
\DeclareMathAccent{\widecheck}{0}{mathx}{"71}

\makeatletter
\newenvironment{breakablealgorithm}
  {
   \begin{center}
     \refstepcounter{algorithm}
     \hrule height.8pt depth0pt \kern2pt
     \renewcommand{\caption}[2][\relax]{
       {\raggedright\textbf{\ALG@name~\thealgorithm} ##2\par}%
       \ifx\relax##1\relax 
         \addcontentsline{loa}{algorithm}{\protect\numberline{\thealgorithm}##2}%
       \else 
         \addcontentsline{loa}{algorithm}{\protect\numberline{\thealgorithm}##1}%
       \fi
       \kern2pt\hrule\kern2pt
     }
  }{
     \kern2pt\hrule\relax
   \end{center}
  }
\makeatother

\title{Towards Optimal Differentially Private Regret Bounds in Linear MDPs}

\NewDocumentCommand{\citetseparate}{>{\SplitList{,}}m}{%
  [\ProcessList{#1}{\citetseparatehelper}\unskip]%
}

\newcommand{\citetseparatehelper}[1]{%
  \hyperref[cite.#1]{\citeauthor{#1}}~\hyperref[cite.#1]{\citeyear{#1}},\hspace{0.5em}%
}

\renewcommand{\citet}[1]{\citeauthor{#1} [\citeyear{#1}]}

\newtheorem{definition}{Definition}
\numberwithin{definition}{section}

\newtheorem{lemma}{Lemma}
\numberwithin{lemma}{section}

\newtheorem{theorem}{Theorem}
\numberwithin{theorem}{section}

\newcommand\norm[1]{\left\lVert#1\right\rVert}

\author{
  Sharan Sahu \\
  Department of Statistics and Data Science\\
  Cornell University\\
  Ithaca, NY 14853 \\
  \texttt{ss4329@cornell.edu} \\
}

\begin{document}

\maketitle

\begin{abstract}
    Motivated by the recent adoption of reinforcement learning (RL) in personalized decision making that relies on using users' sensitive and private information, we study regret minimization in the episodic inhomogeneous linear Markov Decision Process (MDP) setting where the transition probabilities and reward functions are linear with respect to some feature mapping $\boldsymbol{\phi}(s, a)$ under the constraints of differential privacy (DP) and more specifically, a relaxation of DP that is compatible with online-learning settings called joint differential privacy (JDP). Prior work due to \citetseparate{luyo2021differentially} in this setting achieves a rate of $\tilde{O}(\sqrt{d^{3}H^{4}K} + H^{11/5}d^{8/5}K^{3/5} / \epsilon^{2/5})$ and was subsequently improved to $\tilde{O}(\sqrt{d^{3}H^{4}K} + H^{3}d^{5/4}K^{1/2} / \epsilon^{1/2})$ by \citetseparate{ngo2022improved}. This bounds rely on $\tilde{O}(\sqrt{d^{3}H^{4}K})$ dependence, the cost of non-private learning, that arises from the regret achieves by LSVI-UCB \citetseparate{jin2020provably}. Recently, \citetseparate{he2023nearlyminimaxoptimalreinforcement} proposed LSVI-UCB\textsuperscript{++}, a minimax optimal algorithm that achieves regret $\tilde{O}(d\sqrt{H^{3}K})$ for the episodic inhomogeneous linear MDP setting using weighted ridge regression and upper confidence value iteration with a Bernstein-type exploration bonus. Additionally, prior work primarily utilized Hoeffding-type bounds, which are easier to use in analysis but result in suboptimal regret bounds. \citetseparate{qiao2024offline} advanced this area by applying Bernstein-type bounds to more effectively control regret for linear MDPs in the offline setting. Inspired by these works, we design an RL algorithm with differential privacy guarantees in the linear MDP setting by privatizing LSVI-UCB\textsuperscript{++}, utilizing the techniques found in \citetseparate{qiao2024offline}. This algorithm achieves regret $\Tilde{O} \left( d\sqrt{H^3K} + H^{15/4}d^{7/6}K^{1/2} / \epsilon \right)$ which surpasses previous state-of-the-art algorithms for linear MDPs. We also find that theory and simulation suggest that the privacy guarantee comes at (almost) no drop in utility compared to the non-private counterpart. 
\end{abstract}

\section{Introduction}
Reinforcement Learning (RL) has started gaining traction in settings involving personalized decision-making such as precision medicine \citetseparate{yazzourh2024medical, liu2022deep}, user experience adaption \citetseparate{khamaj2024adapting}, recommender systems \citetseparate{afsar2022reinforcement}, and autonomous driving \citetseparate{sallab2017deep}. In such settings, agents learn reasonable policies by learning from potentially private and sensitive user feedback and data. For example, imagine a health-focused mobile application designed to help users adopt a healthier lifestyle by recommending daily activities tailored to their needs and goals. This agent learns an optimal policy by observing user feedback, such as completion rates of recommended activities, user satisfaction ratings, and other behavioral signals. This process inherently involves sensitive data—information that users may consider private, such as their age, weight, location, health habits, and physical activity levels. \citetseparate{hartley2023neural} recently showed that patient information can be memorized by agents even when it occurs on a single training data sample within the dataset. 

To safeguard users’ privacy, it’s essential to incorporate privacy-preserving mechanisms into the RL framework. Differential Privacy (DP) \citetseparate{dwork2006calibrating} has emerged as a rigorous mathematical notion of privacy in algorithms. The guarantee of a differentially private RL algorithm is that its behavior hardly changes when a single individual joins or leaves the dataset. It turns out for problems in this RL setting, the standard definition of DP is too stringent since it necessarily implies that in a setting where a user trusts an central agency with sensitive information in exchange for a service or recommendation, none of the agent’s recommendations could reveal information about the user. Completely eliminating information about a user's data would make it impossible for the agent to make useful recommendations or actions. 

We rely on a Joint Differential Privacy (JDP), a relaxed notion of DP \citetseparate{kearns2015robust}. JDP requires that if any single user changes their data, the information observed by all the other users cannot change substantially and has been adapted in the context of differentially private contextual bandits \citetseparate{DBLP:journals/corr/abs-1810-00068}. There has been a line of literature that attempts to tackle incorporating JDP into RL algorithms in the linear MDP setting. The first work that we are aware of is \citetseparate{luyo2021differentially} who privatize LSVI-UCB \citetseparate{jin2020provably} to get a regret bound of $\tilde{O}(\sqrt{d^{3}H^{4}K} + H^{11/5}d^{8/5}K^{3/5} / \epsilon^{2/5})$. This was subsequently improved to $\tilde{O}(\sqrt{d^{3}H^{4}K} + H^{3}d^{5/4}K^{1/2} / \epsilon^{1/2})$ by \citet{ngo2022improved} through more refined analysis. Both these works rely on self-normalizing martingale concentration bounds, notably Azuma-Hoeffding, for their regret analysis. This allows for the analysis to be simple but results in suboptimal regret bounds. 

Recently, \citetseparate{qiao2024offline} were able to apply self-normalized Bernstein-type martingale bounds with sharper analysis to more effectively control regret for linear MDPs in the offline setting. Additionally, \citetseparate{he2023nearlyminimaxoptimalreinforcement} proposed LSVI-UCB\textsuperscript{++}, a minimax optimal algorithm that achieves regret $\tilde{O}(d\sqrt{H^{3}K})$ for the episodic inhomogeneous linear MDP setting using weighted ridge regression and upper confidence value iteration with a Bernstein-type exploration bonus which improved on LSVI-UCB. 

\textbf{Our contributions}. Inspired by these works, we design an RL algorithm with differential privacy guarantees in the linear MDP setting by privatizing LSVI-UCB\textsuperscript{++}, utilizing the techniques found in \citetseparate{qiao2024offline}.
\begin{itemize}
    \item We propose the DP-LSVI-UCB\textsuperscript{++} algorithm, which achieves a regret bound of 
    $\tilde{O} \left( d\sqrt{H^3K} + H^{15/4}d^{7/6}K^{1/2} / \epsilon \right)$, surpassing the previous state-of-the-art bounds for linear MDPs under JDP constraints.

    \item LSVI-UCB\textsuperscript{++} framework, we integrate private mechanisms such as Gaussian noise and Gaussian Orthogonal Ensemble (GOE) perturbations of Gram matrices, enabling the preservation of privacy while maintaining strong utility guarantees.

    \item Our analysis employs Bernstein-type martingale concentration inequalities, unlike prior approaches relying on Hoeffding-type bounds, leading to tighter and more efficient regret guarantees.

    \item We provide empirical simulations that demonstrate the effectiveness of DP-LSVI-UCB\textsuperscript{++}, showcasing (almost) no drop in utility compared to its non-private counterpart across various privacy budgets.
\end{itemize}

\subsection{Related work}
\textbf{Tabular MDPs}: The intersection of DP and RL has been explored within the context of tabular MDPs. In these cases, DP is often achieved through privatization of visitation counts, which ensures that sensitive trajectory data remains protected. Under the constraint of JDP, \citetseparate{vietri2020private} designed PUCB by privatizing UBEV \citetseparate{dann2017unifying}, and \citetseparate{chowdhury2022differentially} devised Private-UCB-VI
by privatizing UCBVI (with bonus 1) \citetseparate{azar2017minimax}, an algorithm with a minimax optimal regret bound in the tabular MDP setting. However, these works primarily utilized Hoeffding-type bounds, which are easier to use in analysis but result in suboptimal regret bounds. \citetseparate{qiao2024offline} advanced this area by applying Bernstein-type bounds to more effectively control regret, and \citet{qiao2023near} designed DP-UCBVI by privatizing UCBVI with bonus 2 \citetseparate{azar2017minimax}. 

\textbf{Linear Mixture MDPs}: There has been some work in the Linear Mixture MDP setting. Under JDP, \citetseparate{luyo2021differentially} devised JDP-UCRL-VTR by privatizing UCRL-VTR \citetseparate{ayoub2020model} with a regret bound $\tilde{O} ( \sqrt{d^{2}H^{4}K} + H^{9/4}d^{3/4}K^{1/2} / \epsilon^{1/2} )$ where $K$ is the number of episodes. \citetseparate{zhou2022differentially} improved on this bound with Private-LinOpt-VI to guarantee JDP with a regret bound of $\tilde{O} ( \sqrt{d^{2}H^{4}K} + H^{5/2}d^{7/4}K^{1/2} / \epsilon^{1/2} )$. 

\textbf{Linear MDPs}: There has also been some work in linear MDPs. Under JDP, \citetseparate{luyo2021differentially} devised Privacy-Preserving LSVI-UCB Through Batching by privatizing LSVI-UCB \citetseparate{jin2020provably}, and achieved a regret bound of $\tilde{O}(\sqrt{d^{3}H^{4}K} + H^{11/5}d^{8/5}K^{3/5} / \epsilon^{2/5})$ by utilizing standard differential private techniques such as the binary tree mechanism  \citetseparate{DBLP:journals/corr/abs-1810-00068, 10.1145/1806689.1806787, 10.1145/2043621.2043626}  and Gaussian mechanism \citetseparate{10.1561/0400000042}. \citetseparate{ngo2022improved} improved on this bound, achieving $\tilde{O}(\sqrt{d^{3}H^{4}K} + H^{3}d^{5/4}K^{1/2} / \epsilon^{1/2})$ regret by utilizing an adaptive batching schedule to reduce the number of policy updates from polynomial in $K$ to $O(\log(K))$.

\section{Problem Setup}
\textbf{Markov Decision Process}. We will work with the episodic inhomogeneous finite horizon MDP $\mathcal{M} = \{ \mathcal{S}, \mathcal{A}, \{ \mathbb{P}_{h} \}_{h}, \{ r_{h} \}_{h} \}$ where $\mathcal{S}, \mathcal{A}$ is the state and action space respectively, $H \in \mathbb{Z}$ is the length of each episode, $\mathbb{P}_{h} : \mathcal{S} \times \mathcal{A} \rightarrow \Delta(\mathcal{S})$ and $r_{h} : \mathcal{S} \times \mathcal{A} \rightarrow [0, 1]$ are the time-dependent transition probability and deterministic reward function. We assume that $\mathcal{S}$ is a measurable space with possibly infinite number of elements and $\mathcal{A}$ is a finite set. In this setting, the policy is time-dependent and we denote this $\pi = \{ \pi_{1}, \cdots, \pi_{H}  \}$ where $\pi_{h}(s)$ denotes the action the policy takes in state $s$ at timestep $h$. With this, we define the time-dependent value function $V_{h}^{\pi} : \mathcal{S} \rightarrow \mathbb{R}$ as 

\[
    V_{h}^{\pi}(s) = \mathbb{E} \left[ \sum_{t = h}^{H} r_{t} (s_{t}, a_{t}) \mid s_{h} = s, a_{t} \sim \pi_{t}(s_{t}) \right]
\]

for any $s \in \mathcal{S}, \; h \in [H]$. Likewise, we can define the state-action function $Q_{h}^{\pi} : \mathcal{S} \times \mathcal{A} \rightarrow \mathbb{R}$ as 

\[
    Q_{h}^{\pi}(s, a) = \mathbb{E} \left[ \sum_{t = h}^{H} r_{t} (s_{t}, a_{t}) \mid s_{h} = s, \; a_{h} = a, \; a_{t} \sim \pi_{t}(s_{t}) \right]
\]

for any $s, a \in \mathcal{S} \times \mathcal{A}, \; h \in [H]$. Since we are working in a finite episode length and action space, we know that there exists an optimal policy $\pi^{*}$ such that $V_{h}^{*}(s) = \mathrm{sup}_{\pi} V_{h}^{\pi}(s)$ for any $s \in \mathcal{S}, \; h \in [H]$ with Bellman equations

\begin{align*}
    Q_{h}^{*}(s, a) &= r_{h}(s, a) + \mathbb{P}_{h}V_{h+1}^{*}(s, a) \\
    V_{h}^{*}(s, a) &= \mathrm{max}_{a \in \mathcal{A}} Q_{h}^{*}(s, a)
\end{align*}

where $\mathbb{P}_{h}V(s, a) = \mathbb{E}_{s^{\prime} \sim \mathbb{P}_{h}}\left( \cdot | s, a \right) V(s^{\prime})$. We measure the performance of online reinforcement learning algorithms by the regret. The regret of an algorithm is defined as
\[
    \mathrm{Regret}(K) = \sum_{k=1}^{K} \left[ V_{1}^{*}\left( s_{1}^{k} \right) - V_{1}^{\pi_{k}}\left( s_{1}^{k} \right) \right]
\]
where $s_{1}$ is the initial state and $\pi_{k}$ is the agent, both during episode $k$.

\textbf{Linear MDP \citetseparate{jin2020provably}}. A finite-horizon MDP $\mathcal{M} = \{ \mathcal{S}, \mathcal{A}, \{ \mathbb{P}_{h} \}_{h}, \{ r_{h} \}_{h} \}$ is a linear MDP with known feature map $\boldsymbol{\phi}: \mathcal{S} \times \mathcal{A} \rightarrow \mathbb{R}^{d}$ if for any $h \in [H]$, there exists $|\mathcal{S}|$ unknown $d$-dimensional measures $\boldsymbol{\mu}_{h} = \left( \mu_{h}(1), \cdots,  \mu_{h}(|\mathcal{S}|) \right) \in \mathbb{R}^{d \times |\mathcal{S}|}$ and an unknown vector $\boldsymbol{\theta}_{h} \in \mathbb{R}^{d}$ such that for any $(s, a) \in \mathcal{S} \times \mathcal{A}$, we have 

    \[
        \mathbb{P}_{h} \left( \cdot \mid s, a \right) = \langle \boldsymbol{\phi}(s, a), \boldsymbol{\mu}_{h}(\cdot) \rangle,  \; r_{h}(s, a) = \langle \boldsymbol{\phi}(s, a), \boldsymbol{\theta}_{h}(s, a) \rangle
    \]
Without loss of generality, we assume that $|| \boldsymbol{\phi}(s, a) ||_{2} < 1$ and $\mathrm{max} \left( ||\boldsymbol{\mu}_{h} \left( \mathcal{S} \right) ||_{2}, ||\boldsymbol{\theta}_{h}||_{2} \right) \leq \sqrt{d}$ for all $s, a, h \in \mathcal{S} \times \mathcal{A} \times [H]$. 

\subsection{Differential Privacy}
In this work, we are interested in providing a privacy-preserving RL algorithm that incorporates the rigorous notion
of differential privacy (DP). We first revisit the definition of differential privacy

\begin{definition}[Differential Privacy \citetseparate{10.1007/11681878_14}]
    A randomized mechanism $A$ satisfies $(\epsilon, \delta)$-differential privacy if for all neighboring datasets $\mathcal{U}, \mathcal{U}^{\prime}$ that differ by one record and for all event $E$ in the output range

    \[
        \mathbb{P} \left( A(\mathcal{U}) \in E \right) \leq e^{\epsilon} \mathbb{P} \left( A(\mathcal{U}^{\prime}) \in E \right) + \delta
    \]

    When $\delta = 0$, we say that our mechanism satisfies $\epsilon$-pure DP whereas for $\delta > 0$, we say our mechanism satisfies $(\epsilon, \delta)$-DP.
\end{definition}

As we discussed in the introduction, standard DP is too stringent of a framework to work in for the RL setting. Thus, we use JDP as a relaxed but still strong notion of privacy

\begin{definition}[Joint Differential Privacy \citetseparate{kearns2015robust}]
    For any $\epsilon > 0$, a randomized mechanism $A : \mathcal{U} \rightarrow \mathcal{A}^{KH}$ is $\epsilon$-joint differentially private if for any $k \in [K]$, any user sequences $\mathcal{U}, \mathcal{U}^{\prime}$ differing on the $k$-th user and any $E \subset \mathcal{A}^{(K-1)H}$

    \[
        \mathbb{P} \left( A_{-k}(\mathcal{U}) \in E \right) \leq e^{\epsilon} \mathbb{P} \left( A_{-k}(\mathcal{U}^{\prime}) \in E \right)
    \]

    where $A_{-k}(\mathcal{U}) \in E$ denotes the sequence of actions recommended to all users except user $k$ belong to the set $E$.  
\end{definition}

While we state our main results in terms of JDP, we will also use zero-Concentrated DP (zCDP) as a tool in our analysis, since it enables cleaner analysis for privacy composition and the Gaussian mechanism.

\begin{definition}[zCDP \citetseparate{DBLP:journals/corr/DworkR16, DBLP:journals/corr/BunS16}]
    A randomized mechanism $A$ satisfies $\rho$-Zero-Concentrated Differential Privacy ($\rho$-zCDP), if for all neighboring datasets $\mathcal{U}$, $\mathcal{U}^{\prime}$ and all $\alpha \in (1, \infty)$,

    \[
        D_{\alpha} \left( A \left( \mathcal{U} \right) || \; A \left( \mathcal{U}^{\prime} \right) \right) \leq \rho \alpha
    \]

    where $D_{\alpha}$ is the Renyi-divergence \citetseparate{DBLP:journals/corr/abs-1206-2459}
\end{definition}

Any algorithm that satisfies $\rho$-zCDP also satisfies approximate-DP. The following proposition from \citetseparate{DBLP:journals/corr/BunS16} shows how to do the mapping between zCDP and approximate-DP.

\begin{lemma}[Converting zCDP to DP \citetseparate{DBLP:journals/corr/BunS16}] 
\label{lem:zCDP-to-DP}
    If mechanism $A$ satisfies $\rho$-zCDP, then $A$ satisfies $\left( \rho + 2\sqrt{\rho \log \left( 1/\delta \right)}, \delta  \right)$-DP.
\end{lemma}

Another simple and important property of zCDP is that compositions of zCDP mechanisms is also zCDP and any post-processing will not affect the privacy guarantees.

\begin{lemma}[Adaptive composition and Post processing of zCDP \citetseparate{DBLP:journals/corr/BunS16}]
\label{lem:adaptive_comp}
    Let $A: \mathcal{X}^{n} \rightarrow \mathcal{Y}$ and $A^{\prime}:\mathcal{X}^{n} \times \mathcal{Y} \rightarrow \mathcal{Z}$. Suppose $A$ satisfies $\rho$-zCDP and $A^{\prime}$ satisfies $\rho^{\prime}$-zCDP. Define $A^{\prime \prime}: \mathcal{X}^{n} \rightarrow \mathcal{Z}$ to be $A^{\prime \prime} \left( x \right) = A^{\prime} \left(x, A(x) \right)$. Then, $A^{\prime \prime}$ is $\left( \rho + \rho^{\prime} \right)$-zCDP.
\end{lemma}

To apply DP techniques to some mechanism, we must know the sensitivity of the function we want to release. Here we give the definition and the notation we use.

\begin{definition}[$l_{2}$-sensitivity]
    Let $\mathcal{U} \sim \mathcal{U}^{\prime}$ be neighboring datasets. Then the $l_{2}$-sensitivity of a function $f: \mathbb{N}^{\mathcal{X}} \rightarrow \mathbb{R}^{d}$ is 

    \[
        \Delta \left( f \right) = \mathrm{max}_{\mathcal{U} \sim \mathcal{U}^{\prime}} || f\left( \mathcal{U} \right) - f\left( \mathcal{U}^{\prime} \right)||_{2}
    \]
\end{definition}

In our analysis, we use the Gaussian mechanism:

\begin{lemma}[Privacy guarantee of Gaussian mechanism \citetseparate{10.1561/0400000042, DBLP:journals/corr/BunS16}]
\label{lem:gaussian_mech}
    Let $f: \mathbb{N}^{\mathcal{X}} \rightarrow \mathbb{R}^{d}$ be an arbitrary $d$-dimensional function with $l_{2}$ sensitivity $\Delta_{2}$. The Gaussian Mechanism $\mathcal{M}$ with noise level $\sigma$ is given by 

    \[
        \mathcal{M} \left(\mathcal{U} \right) = f \left( \mathcal{U} \right) + \mathcal{N}\left(0, \sigma^{2}I_{d} \right)
    \]

    For any $\rho > 0$, a Gaussian Mechanism with noise parameter $\sigma^{2} = \frac{\Delta_{2}^{2}}{2\rho}$ is $\rho$-zCDP. Additionally, for all $0 < \delta, \epsilon < 1$, a Gaussian Mechanism with noise parameter $\sigma = \frac{\Delta_{2}}{\epsilon} \sqrt{2 \log \left( \frac{1.25}{\delta} \right)}$ satisfies $(\epsilon, \delta)$-DP.
\end{lemma}

Lastly, we use the following lemma to conclude that our algorithm is indeed joint differentially private

\begin{lemma}[Billboard lemma \citetseparate{DBLP:journals/corr/HsuHRRW13}]
\label{lem:billboard_lemma}
    Suppose that a randomized mechanism $A: \mathcal{X}^{n} \rightarrow \mathcal{Y}$ is $(\epsilon, \delta)$-differentially private. Let $U \in \mathcal{U}$ be a dataset containing $n$ users. Then, consider any set of functions $f_{i}: \mathcal{U}_{i} \times \mathcal{Y} \rightarrow \mathcal{Y}_{i}$ for $i \in [n]$ where $\mathcal{U}_{i}$ is the portion of the dataset containing user $i$'s data. Then, the composition $\left\{ f_{i} \left( \Pi_{i}\left(U \right), A \left( U \right)  \right) \right\}_{i \in [n]}$ is $(\epsilon, \delta)$-JDP where $\Pi: \mathcal{U} \rightarrow \mathcal{U}_{i}$ is the canonical projection to the $i$-th user's data.
\end{lemma}

\section{Main results}
We now introduce our RL algorithm for linear MDPs with a JDP guarantee. We will first revisit the non-private version of LSVI-UCB\textsuperscript{++} proposed by \citetseparate{he2023nearlyminimaxoptimalreinforcement} and then we will propose our algorithm along with the techniques used to privatize LSVI-UCB\textsuperscript{++} with a desirable privacy-accuracy tradeoff. 

\textbf{LSVI-UCB\textsuperscript{++}}. To estimate the parameter \( \mu_h \) in linear MDPs, the LSVI-UCB\textsuperscript{++} algorithm employs a weighted ridge regression approach:
\[
\Lambda_{k, h} = \lambda I + \sum_{i=1}^{k-1} \Bar{\sigma}_{i,h}^{-2} \phi(s_h^i, a_h^i) \phi(s_h^i, a_h^i)^\top,
\]
\[
\widehat{w}_{k, h} = (\Lambda_{k, h})^{-1} \sum_{i=1}^{k-1} \Bar{\sigma}_{i,h}^{-2} \phi(s_h^i, a_h^i) \widehat{V}_{k, h+1}(s_{h+1}^{i}),
\]
\[
\widecheck{w}_{k, h} = (\Lambda_{k, h})^{-1} \sum_{i=1}^{k-1} \Bar{\sigma}_{i,h}^{-2} \phi(s_h^i, a_h^i) \widecheck{V}_{k, h+1}(s_{h+1}^{i}),
\]
\[
\overline{w}_{k, h} = \Tilde{\Lambda}_{k, h}^{-1}  \sum_{i=1}^{k-1} \Bar{\sigma}_{i, h}^{-2} \phi \left( s_{h}^{i}, a_{h}^{i} \right)\widehat{V}_{k, h+1}(s_{h+1}^{i})^{2}
\]
where \( \Bar{\sigma}_{i,h} \) represents the variance of the optimal value function and is updated iteratively. This weighting by variance improves estimation accuracy by incorporating information about uncertainty. The optimistic and pessimistic value functions are updated as:
\[
\widehat{Q}_{k, h}(s, a) = r_h(s, a) + \widehat{w}_{k, h}^\top \phi(s, a) + \widehat{\beta} \|\phi(s, a)\|_{\Lambda_{k, h}^{-1}},
\]
\[
\widecheck{Q}_{k, h}(s, a) = r_h(s, a) + \widecheck{w}_{k, h}^\top \phi(s, a) - \widecheck{\beta} \|\phi(s, a)\|_{\Lambda_{k, h}^{-1}},
\]
with the corresponding state-value functions:
\[
\widehat{V}_{k, h}(s) = \max_{a \in \mathcal{A}} \widehat{Q}_{k, h}(s, a), \quad \widecheck{V}_{k, h}(s) = \max_{a \in \mathcal{A}} \widecheck{Q}_{k, h}(s, a).
\]
Here, \(\widehat{\beta}\) and \(\widecheck{\beta}\) determine the exploration bonuses, designed using Bernstein-type bounds. To ensure that the variance used for weighting in the ridge regression remains stable and avoids underestimation, we define a regularized variance for weighing:

\[
    \Bar{\sigma}_{k,h} = \max\{\sigma_{k,h}, H, 2d^3H^2\|\phi(s_k^h, a_k^h)\|_{\Lambda_{k,h}^{-1}}^{1/2}\}
\]

where the estimated variance $\sigma_{k,h}$ of the state-value function is 
\[
    \sigma_{k, h} = \sqrt{\overline{\mathbb{V}}_{k, h}\widehat{V}_{k, h+1} (s_{h}^{k}, a_{h}^{k}) + E_{k, h} + D_{k, h} + H}
\]
where we estimate the variance itself as 
\[
    \overline{\mathbb{V}}_{k, h}\widehat{V}_{k, h+1} (s_{h}^{k}, a_{h}^{k}) = \left[ \overline{w}_{k, h}^{\top} \phi \left( s_{h}^{k}, a_{h}^{k} \right) \right]_{[0, H^{2}]} - \left[ \widehat{w}_{k, h}^{\top} \phi \left( s_{h}^{k}, a_{h}^{k} \right) \right]_{[0, H]}^{2}
\]
Here, $E_{k,h}$ is the error between the estimated variance and
the true variance of $V_{k,h+1}$, and $D_{k,h}$ is the error between the variance of $V_{k,h+1}$ and the variance of the optimal value function $V_{h}^{*}$ \footnote{For more details about the LSVI-UCB\textsuperscript{++}, refer to \citetseparate{he2023nearlyminimaxoptimalreinforcement}}. When privatizing LSVI-UCB\textsuperscript{++}, we aim to privatize the individual statistics involved in making our final estimate $\widehat{w}_{k, h}$. 

\textbf{Private Model Components}. In order to ensure differential privacy, the technique that we commonly employ in differential privacy is to cleverly add noise such that we achieve $\rho$-zCDP, but we also have utility of the specific statistics i.e. the private statistic is close to the non-private statistic with high probability. We add independent Gaussian noise to the $4HK$ statistics in DP-LSVI-UCB\textsuperscript{++} (Algorithm~\ref{alg:dp-lsvi_ucb++}). Then, by the adaptive composition of zCDP (Lemma~\ref{lem:adaptive_comp}), it suffices to ensure that each statistic is $\rho_{0}$-zCDP where $\rho_{0} = \frac{\rho}{4HK}$. In particular, in DP-LSVI-UCB\textsuperscript{++}, we utilize $\phi_{1}, \phi_{2}, \phi_{3}, K_{1}$ to denote the noise that we add. For all $\phi_{i}$, we simply utilize the Gaussian Mechanism (Lemma~\ref{lem:gaussian_mech}). For $K_{1}$, we utilize a recent result by \citetseparate{DBLP:journals/corr/abs-2111-02281} to release the Gram matrix using the GOE perturbations of the form $\frac{1}{\sqrt{2}} \left( Z + Z^{\top} \right)$. We also add $2\Tilde{\lambda}_{\Lambda}$ to ensure that $K_{1}$ remains positive definite as the noise added violates this condition which we require for invertibility. In past literature, many resort to using a binary tree mechanism for privatizing the Gram matrix by recursively partitioning and privatizing partial sums. We find that privatization through GOE is better suited for this setting as it directly exploits their symmetry, yielding tighter utility bounds for the same privacy guarantees. We now present DP-LSVI-UCB\textsuperscript{++} (Algorithm~\ref{alg:dp-lsvi_ucb++})

\begin{breakablealgorithm}
  \caption{DP-LSVI-UCB\textsuperscript{++}}
  \label{alg:dp-lsvi_ucb++}
  \begin{algorithmic}[1]
\REQUIRE Confidence radius $\widehat{\beta}$, $\widecheck{\beta}$, $\Tilde{\beta}$, Budget for zCDP $\rho$, Failure probability $\delta$
\STATE Set $\rho_{0} \leftarrow \frac{\rho}{4HK}$. Sample $\phi_{1}, \phi_{2} \sim \mathcal{N} \left(0, \frac{2H^{2}}{\rho_{0}}I_{d} \right)$, $\phi_{3} \sim \mathcal{N} \left( 0, \frac{2H^{4}}{\rho_{0}}I_{d} \right)$, $K_{1} \leftarrow \frac{1}{\sqrt{2}} \left( Z + Z^{\top} \right)$ where $Z_{i, j} \sim \mathcal{N} \left( 0, \frac{1}{4\rho_{0}} \right)$, $\Tilde{\lambda}_{\Lambda} = O \left( \sqrt{\frac{dHK}{\rho}} \right)$. Initialize $k_{\text{last}} = 0$ and for each stage $h \in [H]$, set $\Tilde{\Lambda}_{0,h}, \Tilde{\Lambda}_{1,h} \leftarrow 2\Tilde{\lambda}_{\Lambda} I$
\STATE For each stage $h \in [H]$ and state-action $(s, a) \in S \times A$, set $\widehat{Q}_{0,h}(s, a) \leftarrow H$, $\widecheck{Q}_{0,h}(s, a) \leftarrow 0$
\FOR{episodes $k = 1, \ldots, K$}
    \STATE Receive the initial state $s_k^1$
    \FOR{stage $h = H, \ldots, 1$}
        \STATE $\Tilde{\widehat{w}}_{k, h} \leftarrow \Tilde{\Lambda}_{k, h}^{-1} \left[ \sum_{i=1}^{k-1} \Tilde{\Bar{\sigma}}_{i, h}^{-2} \phi \left( s_{h}^{i}, a_{h}^{i} \right)\Tilde{\widehat{V}}_{k, h+1}(s_{h+1}^{i}) + \phi_{1} \right]$
        \STATE $\Tilde{\widecheck{w}}_{k, h} \leftarrow \Tilde{\Lambda}_{k, h}^{-1} \left[ \sum_{i=1}^{k-1} \Tilde{\Bar{\sigma}}_{i, h}^{-2} \phi \left( s_{h}^{i}, a_{h}^{i} \right)\Tilde{\widecheck{V}}_{k, h+1}(s_{h+1}^{i}) + \phi_{2} \right]$
        \STATE $\Tilde{\overline{w}}_{k, h} \leftarrow \Tilde{\Lambda}_{k, h}^{-1} \left[ \sum_{i=1}^{k-1} \Tilde{\Bar{\sigma}}_{i, h}^{-2} \phi \left( s_{h}^{i}, a_{h}^{i} \right)\Tilde{\widehat{V}}_{k, h+1}(s_{h+1}^{i})^{2} + \phi_{3} \right]$
        \STATE $\overline{\mathbb{V}}_{k, h}\Tilde{\widehat{V}}_{k, h+1} (s_{h}^{k}, a_{h}^{k}) \leftarrow \left[ \Tilde{\overline{w}}_{k, h}^{\top} \phi \left( s_{h}^{k}, a_{h}^{k} \right) \right]_{[0, H^{2}]} - \left[ \Tilde{\widehat{w}}_{k, h}^{\top} \phi \left( s_{h}^{k}, a_{h}^{k} \right) \right]_{[0, H]}^{2}$
        \IF{there exists a stage $h' \in [H]$ such that $\det(\Tilde{\Lambda}_{k,h'}) \geq 2 \det(\Tilde{\Lambda}_{k_{\text{last}},h'})$}
        \label{alg:det-condition}
            \STATE $\Tilde{\widehat{Q}}_{k, h}(s, a) \leftarrow \mathrm{min} \left\{r_{h}(s, a) +  \Tilde{\widehat{w}}_{k, h}^{\top}\phi \left( s, a \right) + \widehat{\beta} ||\phi \left(s, a \right)||_{\Tilde{\Lambda}_{k,h}^{-1}}, \Tilde{\widehat{Q}}_{k-1, h} \left(s, a \right), H  \right\}$
            \STATE $\Tilde{\widecheck{Q}}_{k, h}(s, a) \leftarrow \mathrm{min} \left\{r_{h}(s, a) +  \Tilde{\widecheck{w}}_{k, h}^{\top}\phi \left( s, a \right) + \widecheck{\beta} ||\phi \left(s, a \right)||_{\Tilde{\Lambda}_{k,h}^{-1}}, \Tilde{\widecheck{Q}}_{k-1, h} \left(s, a \right), 0  \right\}$
            \STATE $k_{\text{last}} \leftarrow k$
        \ELSE
            \STATE $\Tilde{\widehat{Q}}_{k, h}(s, a) = \Tilde{\widehat{Q}}_{k-1, h}(s, a)$
            \STATE $\Tilde{\widecheck{Q}}_{k, h}(s, a) = \Tilde{\widecheck{Q}}_{k-1, h}(s, a)$
        \ENDIF
        \STATE $\Tilde{\widehat{V}}_{k,h}(s) = \max_{a \in \mathcal{A}} \Tilde{\widehat{Q}}_{k, h}(s, a)$
        \STATE $\Tilde{\widecheck{V}}_{k,h}(s) = \max_{a \in \mathcal{A}} \Tilde{\widecheck{Q}}_{k, h}(s, a)$
        \STATE $\Tilde{\pi}_{k, h}(s) = \mathrm{argmax}_{a \in \mathcal{A}} \Tilde{\widehat{Q}}_{k, h}(s, a)$
    \ENDFOR
    \FOR{stage $h = 1, \ldots, H$}
        \STATE Take action $a_k^h \leftarrow \arg\max_a Q_{k,h}(s_k^h, a)$
        \STATE $\Tilde{\sigma}_{k, h} \leftarrow \sqrt{\overline{\mathbb{V}}_{k, h}\Tilde{\widehat{V}}_{k, h+1} (s_{h}^{k}, a_{h}^{k}) + E_{k, h} + D_{k, h} + H}$
        \STATE $\Tilde{\Bar{\sigma}}_{k,h} \leftarrow \max\{\Tilde{\sigma}_{k,h}, H, 2d^3H^2\|\phi(s_k^h, a_k^h)\|_{\Tilde{\Lambda}_{k,h}^{-1}}^{1/2}\}$
        \STATE $\Tilde{\Lambda}_{k+1,h} = \Tilde{\Lambda}_{k,h} + \Tilde{\Bar{\sigma}}_{k,h}^{-2} \phi(s_k^h, a_k^h)\phi(s_k^h, a_k^h)^\top + K_{1}$
        \STATE Receive next state $s_k^{h+1}$
    \ENDFOR
\ENDFOR
\end{algorithmic}
\end{breakablealgorithm}

If one looks at the LSVI-UCB\textsuperscript{++} algorithm and compares it to our algorithm, one will notice that this algorithm is very similar except instead of using the raw statistics, we replace them with private ones as we described above. Since our algorithms are the same, most of the analysis carried out will be similar except that we will use the utility of the privatized statistics. We now present the privacy guarantee of DP-LSVI-UCB\textsuperscript{++}:

\begin{theorem}[Privacy Guarantee]
    \label{thm:privacy}
    DP-LSVI-UCB\textsuperscript{++} (Algorithm~\ref{alg:dp-lsvi_ucb++}) satisfies $(\epsilon, \delta^{\prime})$-JDP.
\end{theorem}

\begin{proof}[Proof of Theorem~\ref{thm:privacy}]
For the full proof, refer to Appendix~\ref{appendix:Privacy Proofs}, particularly Theorem~\ref{appendix:privacy-guarantee}. Note that we use $\delta^{\prime}$ to distinguish between the $\delta^{\prime}$ failure probability of the JDP-mechanism and the $\delta$ high probability bounds we get in our regret analysis. At a high level, we first compute the sensitivity of our privatized statistics. With these sensitivities, we simply use a Gaussian mechanism with sufficient noise using Lemma~\ref{lem:gaussian_mech}. Doing this allows us to show that each of our privatized statistics is $\rho_{0}$-zCDP so by advanced composition (Lemma~\ref{lem:adaptive_comp}), we can conclde that DP-LSVI-UCB\textsuperscript{++} is $\rho$-zCDP. Using Lemma~\ref{lem:zCDP-to-DP}, we can show Algorithm~\ref{alg:dp-lsvi_ucb++} is $(\epsilon, \delta^{\prime})$-DP. Finally, since the actions sent to each user depends on a function constructed with DP and their private data only, we can conclude that DP-LSVI-UCB\textsuperscript{++} is $(\epsilon, \delta^{\prime})$-JDP by the Billboard Lemma (Lemma~\ref{lem:billboard_lemma}).
\end{proof}

\begin{theorem}
    \label{theorem:regret-bound}
    For any linear MDP $\mathcal{M}$, if we set the confidence radii $\widehat{\beta}$, $\widecheck{\beta}$, $\overline{\beta}$ as follows:

\[
    \widehat{\beta} = O\left(HL\sqrt{d \Tilde{\lambda}_{\Lambda}} + \sqrt{d^3 H^2 \log^2 \left(\frac{HK^{4}L^{2}d}{\delta \Tilde{\lambda}_{\Lambda}}\right)} \right),
\]

\[
    \widecheck{\beta} = O\left(HL\sqrt{d \Tilde{\lambda}_{\Lambda}} + \sqrt{d^3 H^2 \log^2 \left(\frac{HK^{4}L^{2}d}{\delta \Tilde{\lambda}_{\Lambda}}\right)} \right),
\]
\[
\bar{\beta} = O\left(H^{2}L^{2} \sqrt{d\Tilde{\lambda}_{\Lambda}} + \sqrt{d^3 H^4 \log^2 \left(\frac{HK^{4}L^{2}d}{\delta \Tilde{\lambda}_{\Lambda}}\right)} \right),
\]
then with high probability of at least $1 - 7\delta$, the regret of DP-LSVI-UCB\textsuperscript{++} is upper bounded as follows:

\[
\text{Regret}(K) \leq \Tilde{O} \left( d\sqrt{H^3K} + \frac{H^{15/4}d^{7/6}K^{1/2} \log(1 / \delta^{\prime})}{\epsilon} \right)
\]

In addition, the number of updates for $\Tilde{\widehat{Q}}_{k,h}$ and $\Tilde{\widecheck{Q}}_{k,h}$ is upper bounded by $O(dH \log(1 + K/d\Tilde{\lambda}_{\Lambda}))$.
\end{theorem}

\begin{proof}[Proof of Theorem~\ref{theorem:regret-bound}]
    For the full proof, refer to Appendix~\ref{appendix:estimated-variance-regret-bound-section}, particularly Lemma~\ref{lemma:regret-bound}. At a high level, we replicate the proofs of \citetseparate{he2023nearlyminimaxoptimalreinforcement} with similar function classes for the optimistic, pessimistic, and squared value functions except using the privatized components. We use these function classes along with standard results for covering numbers to determine the confidence radii (Lemma ~\ref{lemma:confidence-bounds-1}), prove the upper bound of the variance estimator (Lemma ~\ref{lemma:confidence-bounds-2}), prove optimism and pessimism, and condition on these to utilize a Bernstein-bound argument to yield a tighter regret bound. Finally, we state some results from \citetseparate{he2023nearlyminimaxoptimalreinforcement} that hold for our analysis and use these to prove the regret bound.
\end{proof}

\section{Empirical simulations}
We evaluate DP-LSVI-UCB\textsuperscript{++} on a synthetic linear MDP that is described in \citetseparate{DBLP:journals/corr/abs-2106-11960, yin2022nearoptimalofflinereinforcementlearning, qiao2024offline}. In this MDP, we fix the horizon to be $H = 20$. We compare our algorithm compared to LSVI-UCB, their differentially private counterparts proposed by \citetseparate{luyo2021differentially, ngo2022improved}, the non-private LSVI-UCB\textsuperscript{++}, and our algorithm DP-LSVI-UCB\textsuperscript{++} in terms of cumulative regret with a fixed privacy budget. We also compare how our regret scales with varying privacy budgets compared to LSVI-UCB\textsuperscript{++}. We ran the simulation 10 times and
took the average performance.

\begin{figure}[H]
    \centering
    \begin{subfigure}[b]{0.49\linewidth}
        \centering
        \includegraphics[width=\linewidth]{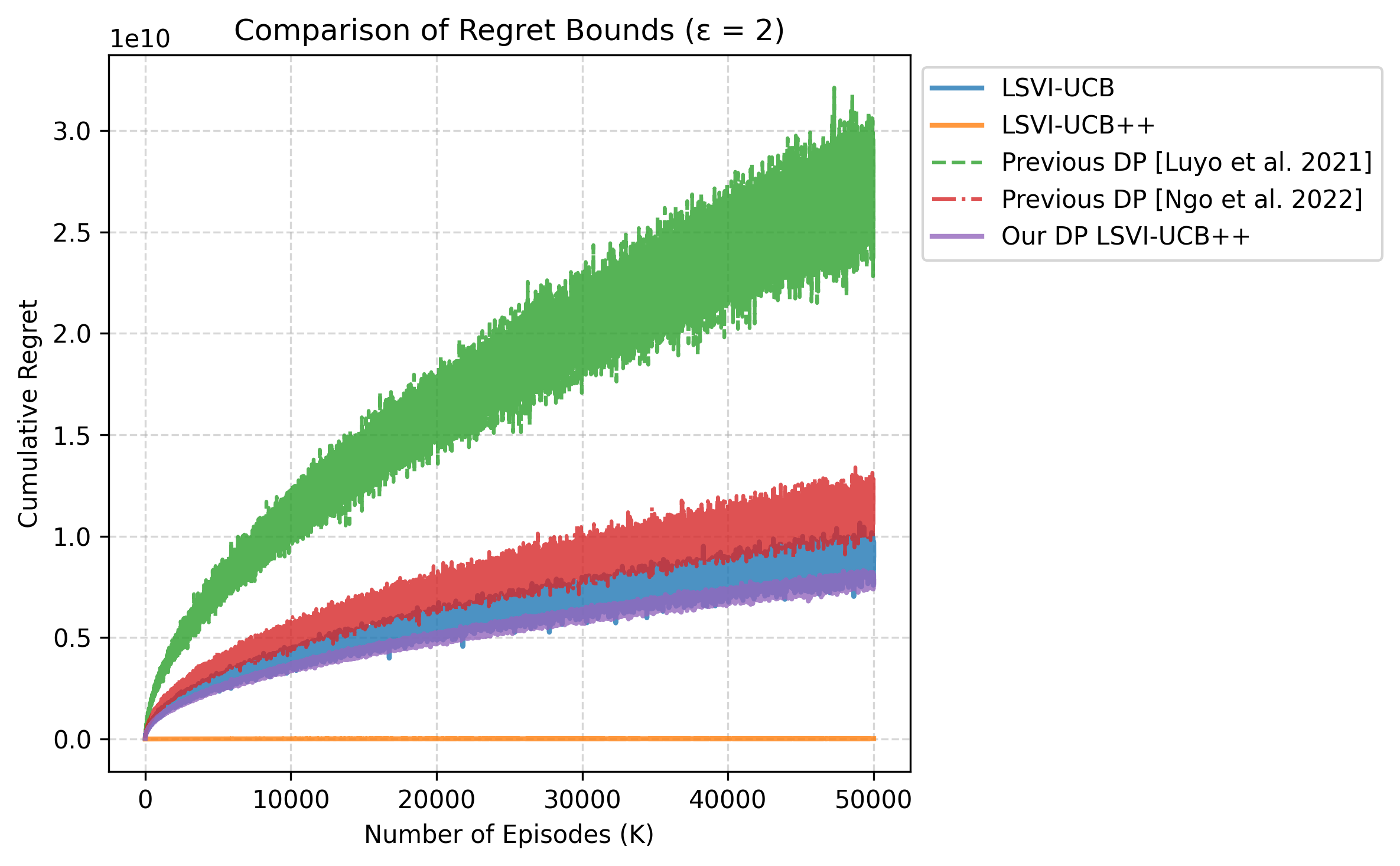}
        \caption{Comparison between different algorithms, $H=20$}
        \label{fig:sub-1}
    \end{subfigure}
    \hfill
    \begin{subfigure}[b]{0.49\linewidth}
        \centering
        \includegraphics[width=\linewidth]{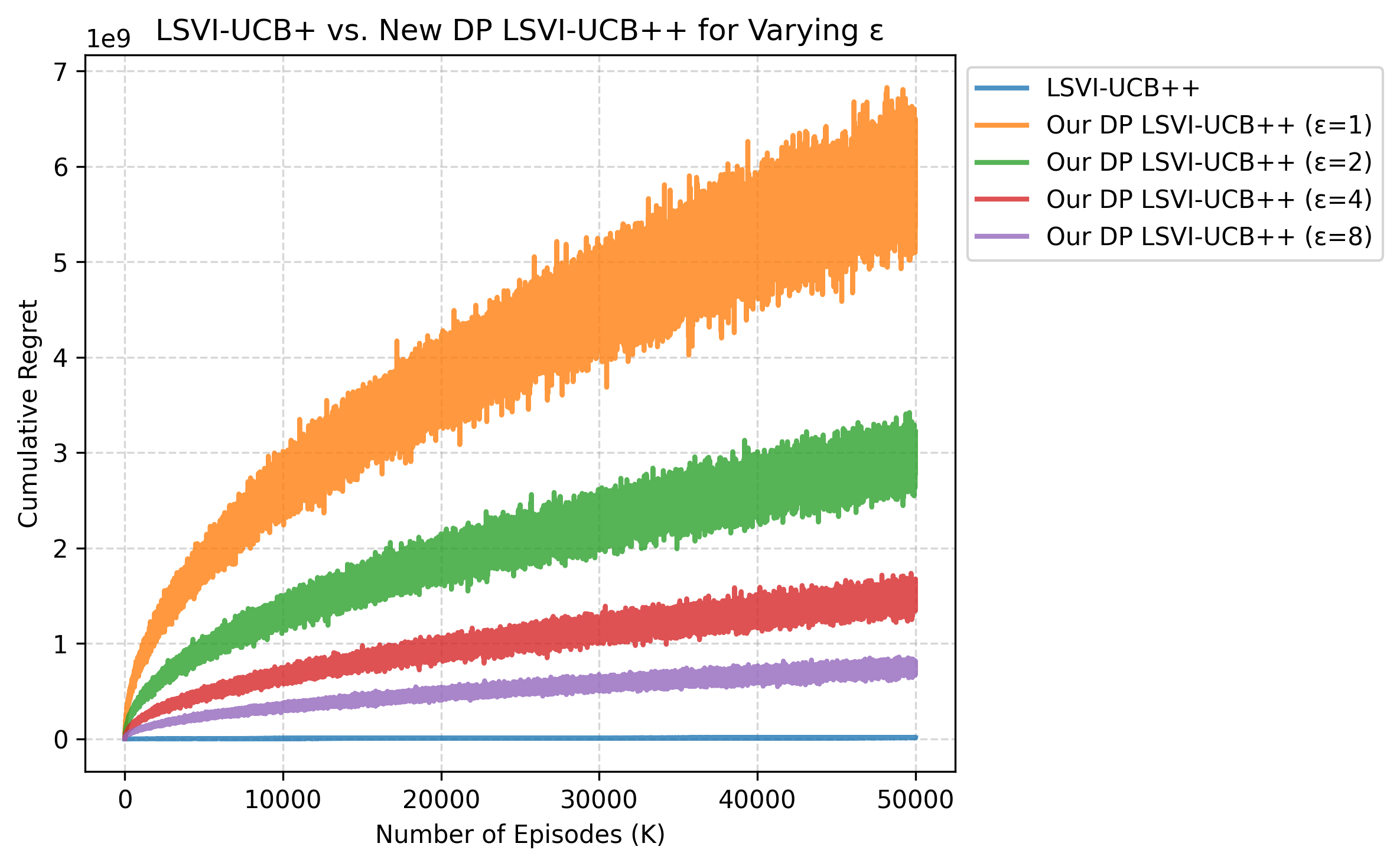}
        \caption{Different privacy budgets, $H=20$}
        \label{fig:sub-2}
    \end{subfigure}
    \label{fig:side-by-side}
\end{figure}
\textbf{Key Takeaways}. From Figure~\ref{fig:sub-1}, we can observe the DP-LSVI-UCB\textsuperscript{++} performs better compared to the previous state-of-the-art algorithm devised by \citetseparate{ngo2022improved} and ofcourse also performs better than \citetseparate{luyo2021differentially}. Additionally, we see that DP-LSVI-UCB\textsuperscript{++} even being a privatized algorithm, performs better than the non-private LSVI-UCB. Looking at Figure~\ref{fig:sub-2}, we see that as we increase the privacy budget of DP-LSVI-UCB\textsuperscript{++}, we get closer to LSVI-UCB\textsuperscript{++} and thus with sufficient noise, we can guarantee $(\epsilon, \delta)$-JDP that will perform slightly worse than LSVI-UCB\textsuperscript{++}. This is due to the fact that we add Gaussian noise to each count. In particular, we enjoy a better regret bound by using the GOE technique as previous state-of-the-art bounds using a binary-tree mechanism that yields suboptimal regret for the same privacy guarantee. We also enjoy a better bound due to our usage of rare-switching to reduce the amount of noise we added (as adding noise to every statistic would lead to suboptimal regret). This also supports our theoretical regret bound since the cost of privacy appears as lower order terms in the regret bound. 

\section{Conclusions and future works}
In this work, we introduced DP-LSVI-UCB\textsuperscript{++}, a differentially private reinforcement learning algorithm for the linear MDP setting, achieving state-of-the-art regret bounds under joint differential privacy (JDP) constraints. Our approach incorporates advanced techniques such as Bernstein-type martingale concentration inequalities and GOE perturbations, enabling us to improve the utility-privacy tradeoff while maintaining strong theoretical guarantees. The algorithm's regret bound $\Tilde{O} \left( d\sqrt{H^3K} + H^{15/4}d^{7/6}K^{1/2} / \epsilon \right)$ surpasses prior works and demonstrates that incorporating differential privacy need not lead to substantial utility degradation. Through empirical simulations, we verified that DP-LSVI-UCB\textsuperscript{++} achieves near-optimal performance, often matching or even outperforming non-private baselines. We believe that there are many promising directions to explore from our study. While our work focuses on linear MDPs, it would be valuable to extend these techniques to the low-rank MDP setting. Additionally, our work utilized Gaussian mechanisms and GOE-based perturbations for privacy guarantees. Exploring alternative mechanisms that adapt noise levels dynamically based on the observed data's sensitivity could lead to improved regret bounds.

\bibliographystyle{plainnat}
\bibliography{citations}

\begin{thebibliography}{36}
\providecommand{\natexlab}[1]{#1}
\providecommand{\url}[1]{\texttt{#1}}
\expandafter\ifx\csname urlstyle\endcsname\relax
  \providecommand{\doi}[1]{doi: #1}\else
  \providecommand{\doi}{doi: \begingroup \urlstyle{rm}\Url}\fi

\bibitem[Abbasi-Yadkori et~al.(2011)Abbasi-Yadkori, P\'{a}l, and Szepesv\'{a}ri]{10.5555/2986459.2986717}
Yasin Abbasi-Yadkori, D\'{a}vid P\'{a}l, and Csaba Szepesv\'{a}ri.
\newblock Improved algorithms for linear stochastic bandits.
\newblock In \emph{Proceedings of the 24th International Conference on Neural Information Processing Systems}, NIPS'11, page 2312–2320, Red Hook, NY, USA, 2011. Curran Associates Inc.
\newblock ISBN 9781618395993.

\bibitem[Afsar et~al.(2022)Afsar, Crump, and Far]{afsar2022reinforcement}
M~Mehdi Afsar, Trafford Crump, and Behrouz Far.
\newblock Reinforcement learning based recommender systems: A survey.
\newblock \emph{ACM Computing Surveys}, 55\penalty0 (7):\penalty0 1--38, 2022.

\bibitem[Ayoub et~al.(2020)Ayoub, Jia, Szepesvari, Wang, and Yang]{ayoub2020model}
Alex Ayoub, Zeyu Jia, Csaba Szepesvari, Mengdi Wang, and Lin Yang.
\newblock Model-based reinforcement learning with value-targeted regression.
\newblock In \emph{International Conference on Machine Learning}, pages 463--474. PMLR, 2020.

\bibitem[Azar et~al.(2017)Azar, Osband, and Munos]{azar2017minimax}
Mohammad~Gheshlaghi Azar, Ian Osband, and R{\'e}mi Munos.
\newblock Minimax regret bounds for reinforcement learning.
\newblock In \emph{International conference on machine learning}, pages 263--272. PMLR, 2017.

\bibitem[Bun and Steinke(2016)]{DBLP:journals/corr/BunS16}
Mark Bun and Thomas Steinke.
\newblock Concentrated differential privacy: Simplifications, extensions, and lower bounds.
\newblock \emph{CoRR}, abs/1605.02065, 2016.
\newblock URL \url{http://arxiv.org/abs/1605.02065}.

\bibitem[Cesa-Bianchi and Lugosi(2006)]{10.5555/1137817}
Nicolo Cesa-Bianchi and Gabor Lugosi.
\newblock \emph{Prediction, Learning, and Games}.
\newblock Cambridge University Press, USA, 2006.
\newblock ISBN 0521841089.

\bibitem[Chan et~al.(2011)Chan, Shi, and Song]{10.1145/2043621.2043626}
T.-H.~Hubert Chan, Elaine Shi, and Dawn Song.
\newblock Private and continual release of statistics.
\newblock \emph{ACM Trans. Inf. Syst. Secur.}, 14\penalty0 (3), November 2011.
\newblock ISSN 1094-9224.
\newblock \doi{10.1145/2043621.2043626}.
\newblock URL \url{https://doi.org/10.1145/2043621.2043626}.

\bibitem[Chowdhury and Zhou(2022)]{chowdhury2022differentially}
Sayak~Ray Chowdhury and Xingyu Zhou.
\newblock Differentially private regret minimization in episodic markov decision processes.
\newblock In \emph{Proceedings of the AAAI Conference on Artificial Intelligence}, volume~36, pages 6375--6383, 2022.

\bibitem[Dann et~al.(2017)Dann, Lattimore, and Brunskill]{dann2017unifying}
Christoph Dann, Tor Lattimore, and Emma Brunskill.
\newblock Unifying pac and regret: Uniform pac bounds for episodic reinforcement learning.
\newblock \emph{Advances in Neural Information Processing Systems}, 30, 2017.

\bibitem[Dwork and Roth(2014)]{10.1561/0400000042}
Cynthia Dwork and Aaron Roth.
\newblock The algorithmic foundations of differential privacy.
\newblock \emph{Found. Trends Theor. Comput. Sci.}, 9\penalty0 (3–4):\penalty0 211–407, August 2014.
\newblock ISSN 1551-305X.
\newblock \doi{10.1561/0400000042}.
\newblock URL \url{https://doi.org/10.1561/0400000042}.

\bibitem[Dwork and Rothblum(2016)]{DBLP:journals/corr/DworkR16}
Cynthia Dwork and Guy~N. Rothblum.
\newblock Concentrated differential privacy.
\newblock \emph{CoRR}, abs/1603.01887, 2016.
\newblock URL \url{http://arxiv.org/abs/1603.01887}.

\bibitem[Dwork et~al.(2006{\natexlab{a}})Dwork, McSherry, Nissim, and Smith]{10.1007/11681878_14}
Cynthia Dwork, Frank McSherry, Kobbi Nissim, and Adam Smith.
\newblock Calibrating noise to sensitivity in private data analysis.
\newblock In \emph{Proceedings of the Third Conference on Theory of Cryptography}, TCC'06, page 265–284, Berlin, Heidelberg, 2006{\natexlab{a}}. Springer-Verlag.
\newblock ISBN 3540327312.
\newblock \doi{10.1007/11681878_14}.
\newblock URL \url{https://doi.org/10.1007/11681878_14}.

\bibitem[Dwork et~al.(2006{\natexlab{b}})Dwork, McSherry, Nissim, and Smith]{dwork2006calibrating}
Cynthia Dwork, Frank McSherry, Kobbi Nissim, and Adam Smith.
\newblock Calibrating noise to sensitivity in private data analysis.
\newblock In \emph{Theory of Cryptography: Third Theory of Cryptography Conference, TCC 2006, New York, NY, USA, March 4-7, 2006. Proceedings 3}, pages 265--284. Springer, 2006{\natexlab{b}}.

\bibitem[Dwork et~al.(2010)Dwork, Naor, Pitassi, and Rothblum]{10.1145/1806689.1806787}
Cynthia Dwork, Moni Naor, Toniann Pitassi, and Guy~N. Rothblum.
\newblock Differential privacy under continual observation.
\newblock In \emph{Proceedings of the Forty-Second ACM Symposium on Theory of Computing}, STOC '10, page 715–724, New York, NY, USA, 2010. Association for Computing Machinery.
\newblock ISBN 9781450300506.
\newblock \doi{10.1145/1806689.1806787}.
\newblock URL \url{https://doi.org/10.1145/1806689.1806787}.

\bibitem[Hartley et~al.(2023)Hartley, Sanchez, Haider, and Tsaftaris]{hartley2023neural}
John Hartley, Pedro~P Sanchez, Fasih Haider, and Sotirios~A Tsaftaris.
\newblock Neural networks memorise personal information from one sample.
\newblock \emph{Scientific Reports}, 13\penalty0 (1):\penalty0 21366, 2023.

\bibitem[He et~al.(2023)He, Zhao, Zhou, and Gu]{he2023nearlyminimaxoptimalreinforcement}
Jiafan He, Heyang Zhao, Dongruo Zhou, and Quanquan Gu.
\newblock Nearly minimax optimal reinforcement learning for linear markov decision processes, 2023.
\newblock URL \url{https://arxiv.org/abs/2212.06132}.

\bibitem[Hsu et~al.(2013)Hsu, Huang, Roth, Roughgarden, and Wu]{DBLP:journals/corr/HsuHRRW13}
Justin Hsu, Zhiyi Huang, Aaron Roth, Tim Roughgarden, and Zhiwei~Steven Wu.
\newblock Private matchings and allocations.
\newblock \emph{CoRR}, abs/1311.2828, 2013.
\newblock URL \url{http://arxiv.org/abs/1311.2828}.

\bibitem[Jin et~al.(2020)Jin, Yang, Wang, and Jordan]{jin2020provably}
Chi Jin, Zhuoran Yang, Zhaoran Wang, and Michael~I Jordan.
\newblock Provably efficient reinforcement learning with linear function approximation.
\newblock In \emph{Conference on learning theory}, pages 2137--2143. PMLR, 2020.

\bibitem[Kearns et~al.(2015)Kearns, Pai, Rogers, Roth, and Ullman]{kearns2015robust}
Michael Kearns, Mallesh~M Pai, Ryan Rogers, Aaron Roth, and Jonathan Ullman.
\newblock Robust mediators in large games.
\newblock \emph{arXiv preprint arXiv:1512.02698}, 2015.

\bibitem[Khamaj and Ali(2024)]{khamaj2024adapting}
Abdulrahman Khamaj and Abdulelah~M Ali.
\newblock Adapting user experience with reinforcement learning: Personalizing interfaces based on user behavior analysis in real-time.
\newblock \emph{Alexandria Engineering Journal}, 95:\penalty0 164--173, 2024.

\bibitem[Liu et~al.(2022)Liu, Shen, and Pan]{liu2022deep}
Mingyang Liu, Xiaotong Shen, and Wei Pan.
\newblock Deep reinforcement learning for personalized treatment recommendation.
\newblock \emph{Statistics in medicine}, 41\penalty0 (20):\penalty0 4034--4056, 2022.

\bibitem[Luyo et~al.(2021)Luyo, Garcelon, Lazaric, and Pirotta]{luyo2021differentially}
Paul Luyo, Evrard Garcelon, Alessandro Lazaric, and Matteo Pirotta.
\newblock Differentially private exploration in reinforcement learning with linear representation.
\newblock \emph{arXiv preprint arXiv:2112.01585}, 2021.

\bibitem[Min et~al.(2021)Min, Wang, Zhou, and Gu]{DBLP:journals/corr/abs-2106-11960}
Yifei Min, Tianhao Wang, Dongruo Zhou, and Quanquan Gu.
\newblock Variance-aware off-policy evaluation with linear function approximation.
\newblock \emph{CoRR}, abs/2106.11960, 2021.
\newblock URL \url{https://arxiv.org/abs/2106.11960}.

\bibitem[Ngo et~al.(2022)Ngo, Vietri, and Wu]{ngo2022improved}
Dung Daniel~T Ngo, Giuseppe Vietri, and Steven Wu.
\newblock Improved regret for differentially private exploration in linear mdp.
\newblock In \emph{International Conference on Machine Learning}, pages 16529--16552. PMLR, 2022.

\bibitem[Qiao and Wang(2023)]{qiao2023near}
Dan Qiao and Yu-Xiang Wang.
\newblock Near-optimal differentially private reinforcement learning.
\newblock In \emph{International Conference on Artificial Intelligence and Statistics}, pages 9914--9940. PMLR, 2023.

\bibitem[Qiao and Wang(2024)]{qiao2024offline}
Dan Qiao and Yu-Xiang Wang.
\newblock Offline reinforcement learning with differential privacy.
\newblock \emph{Advances in Neural Information Processing Systems}, 36, 2024.

\bibitem[Redberg and Wang(2021)]{DBLP:journals/corr/abs-2111-02281}
Rachel Redberg and Yu{-}Xiang Wang.
\newblock Privately publishable per-instance privacy.
\newblock \emph{CoRR}, abs/2111.02281, 2021.
\newblock URL \url{https://arxiv.org/abs/2111.02281}.

\bibitem[Sallab et~al.(2017)Sallab, Abdou, Perot, and Yogamani]{sallab2017deep}
Ahmad~EL Sallab, Mohammed Abdou, Etienne Perot, and Senthil Yogamani.
\newblock Deep reinforcement learning framework for autonomous driving.
\newblock \emph{arXiv preprint arXiv:1704.02532}, 2017.

\bibitem[Shariff and Sheffet(2018)]{DBLP:journals/corr/abs-1810-00068}
Roshan Shariff and Or~Sheffet.
\newblock Differentially private contextual linear bandits.
\newblock \emph{CoRR}, abs/1810.00068, 2018.
\newblock URL \url{http://arxiv.org/abs/1810.00068}.

\bibitem[van Erven and Harremo{\"{e}}s(2012)]{DBLP:journals/corr/abs-1206-2459}
Tim van Erven and Peter Harremo{\"{e}}s.
\newblock R{\'{e}}nyi divergence and kullback-leibler divergence.
\newblock \emph{CoRR}, abs/1206.2459, 2012.
\newblock URL \url{http://arxiv.org/abs/1206.2459}.

\bibitem[Vietri et~al.(2020)Vietri, Balle, Krishnamurthy, and Wu]{vietri2020private}
Giuseppe Vietri, Borja Balle, Akshay Krishnamurthy, and Steven Wu.
\newblock Private reinforcement learning with pac and regret guarantees.
\newblock In \emph{International Conference on Machine Learning}, pages 9754--9764. PMLR, 2020.

\bibitem[Wang et~al.(2021)Wang, Zhou, and Gu]{DBLP:journals/corr/abs-2101-02195}
Tianhao Wang, Dongruo Zhou, and Quanquan Gu.
\newblock Provably efficient reinforcement learning with linear function approximation under adaptivity constraints.
\newblock \emph{CoRR}, abs/2101.02195, 2021.
\newblock URL \url{https://arxiv.org/abs/2101.02195}.

\bibitem[Yazzourh et~al.(2024)Yazzourh, Savy, Saint-Pierre, and Kosorok]{yazzourh2024medical}
Sophia Yazzourh, Nicolas Savy, Philippe Saint-Pierre, and Michael~R Kosorok.
\newblock Medical knowledge integration into reinforcement learning algorithms for dynamic treatment regimes.
\newblock \emph{arXiv preprint arXiv:2407.00364}, 2024.

\bibitem[Yin et~al.(2022)Yin, Duan, Wang, and Wang]{yin2022nearoptimalofflinereinforcementlearning}
Ming Yin, Yaqi Duan, Mengdi Wang, and Yu-Xiang Wang.
\newblock Near-optimal offline reinforcement learning with linear representation: Leveraging variance information with pessimism, 2022.
\newblock URL \url{https://arxiv.org/abs/2203.05804}.

\bibitem[Zhou and Gu(2022)]{zhou2022computationallyefficienthorizonfreereinforcement}
Dongruo Zhou and Quanquan Gu.
\newblock Computationally efficient horizon-free reinforcement learning for linear mixture mdps, 2022.
\newblock URL \url{https://arxiv.org/abs/2205.11507}.

\bibitem[Zhou(2022)]{zhou2022differentially}
Xingyu Zhou.
\newblock Differentially private reinforcement learning with linear function approximation.
\newblock \emph{Proceedings of the ACM on Measurement and Analysis of Computing Systems}, 6\penalty0 (1):\penalty0 1--27, 2022.

\end{thebibliography}

\appendix

\section{Privacy Proofs}
\label{appendix:Privacy Proofs}
First, we will prove the main privacy guarantee of our algorithm

\begin{theorem}[Privacy Guarantee]
    \label{appendix:privacy-guarantee}
    DP-LSVI-UCB\textsuperscript{++} (Algorithm~\ref{alg:dp-lsvi_ucb++}) satisfies $(\epsilon, \delta)$-JDP.
\end{theorem}

\begin{proof}[Proof of Theorem~\ref{appendix:privacy-guarantee}]
        In order to prove this, we must first determine the $l_{2}$ sensitivity of our privatized statistics. Consider two neighboring user sequences $\mathcal{U}$, $\mathcal{U}^{\prime}$. Let $i \leq k$ be some episode where $s_{h}^{i} \neq s_{h}^{\prime i}$ and $a_{h}^{i} \neq a_{h}^{\prime i}$ where $\left( s_{h}^{i}, a_{h}^{i} \right) \in \mathcal{U}$ and $\left( s_{h}^{\prime i}, a_{h}^{\prime i} \right) \in \mathcal{U}^{\prime}$. Then, $\sum_{i=1}^{k-1} \Tilde{\Bar{\sigma}}_{i, h}^{-2} \phi \left( s_{h}^{i}, a_{h}^{i} \right)\Tilde{\widehat{V}}_{k, h+1}(s_{h+1}^{i})$ and $\sum_{i=1}^{k-1} \Tilde{\Bar{\sigma}}_{i, h}^{-2} \phi \left( s_{h}^{i}, a_{h}^{i} \right)\Tilde{\widecheck{V}}_{k, h+1}(s_{h+1}^{i})$ have $l_{2}$ sensitivity of $2H$. Likewise, $\sum_{i=1}^{k-1} \Tilde{\Bar{\sigma}}_{i, h}^{-2} \phi \left( s_{h}^{i}, a_{h}^{i} \right)\Tilde{\widehat{V}}_{k, h+1}(s_{h+1}^{i})^{2}$ has $l_{2}$ sensitivity of $2H^{2}$. Thus, by using a Gaussian mechanism with noise $\sigma^{2} = \frac{2H^{2}}{\rho_{0}}$ and $\sigma^{2} = \frac{2H^{4}}{\rho_{0}}$, respectively, we are guaranteed to have $\rho_{0}$-zCDP for each of the first three terms (Lemma~\ref{lem:gaussian_mech}). For the term $\sum_{i=1}^{k-1} \Tilde{\Bar{\sigma}}_{i,h}^{-2} \phi(s_i^h, a_i^h)\phi(s_i^h, a_i^h)^\top$, according to Appendix D in \citetseparate{DBLP:journals/corr/abs-2111-02281}, we have that the per-instance $l_{2}$ sensitivity is given as

\[
    ||\Delta_{x}||_{2} = \frac{1}{\sqrt{2}} \mathrm{sup}_{\phi: ||\phi||_{2} \leq 1} ||\phi \phi^{\top}||_{F} \leq \frac{1}{\sqrt{2}}
\]

Thus, by using a Gaussian mechanism with noise $\sigma^{2} = \frac{1}{4\rho_{0}}$, we guarantee that this statistic is $\rho_{0}$-zCDP. \footnote{For those more interested in the details of the GOE DP mechanism, we refer the reader to Appendix D of \citetseparate{DBLP:journals/corr/abs-2111-02281}}. Now, we need to track $4KH$ statistics so combing the results of each privatized statistic advanced composition (Lemma~\ref{lem:adaptive_comp}) to conclude that the DP-LSVI-UCB\textsuperscript{++} is $\rho$-zCDP. Thus, by conversion of zCDP to DP (Lemma~\ref{lem:zCDP-to-DP}), Algorithm~\ref{alg:dp-lsvi_ucb++} satisfies $\left( \rho + 2\sqrt{\rho \log \left( 1/\delta \right)}, \delta  \right)$-DP. Since the actions sent to each user depends on a function constructed with DP and their private data only, by the Billboard Lemma (Lemma~\ref{lem:billboard_lemma}), we conclude Algorithm~\ref{alg:dp-lsvi_ucb++} is $(\epsilon, \delta)$-JDP. 
\end{proof}

Now, we will give a high probability bound of the noises we add for privatization. These will be useful for the further analysis we do later on.

\begin{lemma}[Utility Analysis]
\label{lemma:private-utility-analysis}
Let 

\[
    L = 4H \sqrt{\frac{dHK}{\rho} \log \left( \frac{10dKH}{\delta} \right)}
\]
and 
\[
    \Tilde{\lambda}_{\Lambda} = \sqrt{\frac{8dHK}{\rho}} \left( 2 + \left( \frac{\log \left( 5c_{1} H / \delta \right)}{c_{2}d}  \right)^{\frac{2}{3}} \right)
\]

for some universal constants $c_{1}, c_{2}$. Then, with probability atleast $1 - \delta$, for all $h, k \in [H] \times [K]$, we have that $||\phi_{1}||_{2} \leq L$, $||\phi_{2}||_{2} \leq L$, and $||\phi_{3}||_{2} \leq HL$. Additionally, we have that $K_{1}$ is symmetric and positive definite with $||K_{1}||_{2} \leq \Tilde{\lambda}_{\Lambda}$.  
\end{lemma}

\begin{proof}[Proof of Lemma~\ref{lemma:private-utility-analysis}]

The bounds on $\phi_{i}$ hold by simple Gaussian concentration and union bound over all $h, k \in [H] \times [K]$. The bound on $K_{1}$ hold from Lemma 19 in \citetseparate{DBLP:journals/corr/abs-2111-02281}.  
\end{proof}

\section{Upper Confidence Bound Proofs}
\label{appendix:ucb-proofs}

We provide this lemma from \citetseparate{jin2020provably} that we will utilize in our analysis.

\begin{lemma}[Lemma D.1 from \citetseparate{he2023nearlyminimaxoptimalreinforcement}, Lemma D.4 from \citetseparate{jin2020provably} for weighted linear regression]
\label{lemma:martingale-ineq-weighted-linear}
    Let $\{x_k\}_{k=1}^\infty$ be a real-valued stochastic process on state space $S$ with corresponding filtration $\{\mathcal{F}_k\}_{k=1}^\infty$. Let $\{\phi_k\}_{k=1}^\infty$ be an $\mathbb{R}^d$-valued stochastic process, where $\phi_k \in \mathcal{F}_{k-1}$ and $\|\phi_k\|_2 \leq 1$. Let $\{w_k\}_{k=1}^\infty$ be a real-valued stochastic process where $w_k \in \mathcal{F}_{k-1}$ and $0 \leq w_k \leq C$. For any $k \geq 0$, define $\Sigma_k = 2\Tilde{\lambda}_{\Lambda} I + \sum_{i=1}^k w_i^2 \phi_i \phi_i^\top + K_{1}$. Then with probability at least $1 - \delta$, for all $k \in \mathbb{N}$ and all functions $V \in \mathcal{V}$ with $\max_s |V(x)| \leq H$, we have
\[
\left\|
\sum_{i=1}^k w_i^2 \phi_i \left\{ V(x_i) - \mathbb{E}[V(x_i) | \mathcal{F}_{i-1}] \right\}
\right\|_{\Sigma_k^{-1}}^2
\leq 4C^2H^2 \left[ \frac{d}{2} \log\left(1 + \frac{kC^2}{\Tilde{\lambda}_{\Lambda}}\right) + \log\left(\frac{N_\varepsilon}{\delta}\right) \right]
+ \frac{8k^2C^4\varepsilon^2}{\Tilde{\lambda}_{\Lambda}},
\]
where $N_\varepsilon$ is the $\varepsilon$-covering number of the function class $\mathcal{V}$ with respect to the distance function $\text{dist}(V_1, V_2) = \max_s |V_1(s) - V_2(s)|$.

\end{lemma}

\begin{proof}[Proof of Lemma~\ref{lemma:martingale-ineq-weighted-linear}]
For any function $V \in \mathcal{V}$, there exists some $\Tilde{V}$ in the $\varepsilon$-net such that $\mathrm{dist} \left(V, \Tilde{V} \right) \leq \varepsilon$. Using this, the concentration error can be upper bounded as 

\begin{align*}
    \norm{\sum_{i=1}^k w_i^2 \phi_i \left\{ V(x_i) - \mathbb{E}[V(x_i) | \mathcal{F}_{i-1}] \right\}}_{\Sigma_k^{-1}}^{2} &\leq 2 \norm{\sum_{i=1}^k w_i^2 \phi_i \left\{ \Tilde{V}(x_i) - \mathbb{E}[\Tilde{V}(x_i) | \mathcal{F}_{i-1}] \right\}}_{\Sigma_k^{-1}}^{2} \\
    &+ 2\norm{\sum_{i=1}^k w_i^2 \phi_i \left\{ \Delta_{V}(x_i) - \mathbb{E}[\Delta_{V}(x_i) | \mathcal{F}_{i-1}] \right\}}_{\Sigma_k^{-1}}^{2}
\end{align*}
where $\Delta_{V} = V - \Tilde{V}$ and the inequality holds from the fact that $\norm{a + b}_{\Sigma}^{2} \leq 2\norm{a}_{\Sigma}^{2} + 2\norm{b}_{\Sigma}^{2}$. For any fixed value function $\Tilde{V}$, take $x_{i} = w_{i}\phi_{i}$ and $\eta_{i} = w_{i}\Tilde{V}(x_{i}) - w_{i}\mathbb{E} \left[ \Tilde{V}(x_{i}) \right]$. Notice that 
\begin{align*}
    &\norm{x_{i}}_{2} \leq C \\
    & \mathbb{E} \left[ \eta_{i} \mid \mathcal{F}_{i} \right] = 0, \; \left| \eta_{i} \right| \leq HC
\end{align*}
Then, by Lemma~\ref{lemma:abbassi-confidence-theorem-2} and taking a union-bound over the $\varepsilon$-net $\mathcal{N}_{\varepsilon}$, we get the first term being upper bounded as 

\begin{align*}
    2 \norm{\sum_{i=1}^k w_i^2 \phi_i \left\{ \Tilde{V}(x_i) - \mathbb{E}[\Tilde{V}(x_i) | \mathcal{F}_{i-1}] \right\}}_{\Sigma_k^{-1}}^{2} \leq 4H^{2}C^{2} \left[ \frac{d}{2} \log \left( 1 + KC^{2} /  \Tilde{\lambda}_{\Lambda} \right) + \log \frac{N_{\varepsilon}}{\delta} \right]
\end{align*}

The second term can be upper bounded as 

\begin{align*}
    2\norm{\sum_{i=1}^k w_i^2 \phi_i \left\{ \Delta_{V}(x_i) - \mathbb{E}[\Delta_{V}(x_i) | \mathcal{F}_{i-1}] \right\}}_{\Sigma_k^{-1}}^{2} &\leq 2k\sum_{i=1}^k \norm{w_i^2 \phi_i \left\{ \Delta_{V}(x_i) - \mathbb{E}[\Delta_{V}(x_i) | \mathcal{F}_{i-1}] \right\}}_{\Sigma_k^{-1}}^{2} \\
    &\leq 8k^{2}C^{4}\varepsilon^{2} / \Tilde{\lambda}_{\Lambda}
\end{align*}
where the first inequality holds from Cauchy-Schwartz and the last inequality holds from $\left| \Delta_{V} \right| \leq \varepsilon$, $w_{i}^{2} \leq C^{2}$, and $\Sigma_{k} \succeq \Tilde{\lambda}_{\Lambda}$. Thus, putting these together, we get the claim.
    
\end{proof}

Now, we are ready to begin to derive our confidence radii. This is a Hoeffding-type upper bound for the estimation
error.

\begin{lemma}
    \label{lemma:confidence-bounds-1}
    Define \( \mathcal{E} \) as the event that the following inequalities hold for all \( s, a, k, h \in \mathcal{S} \times \mathcal{A} \times [K] \times [H] \):
\[
\left| \Tilde{\widehat{w}}_{k,h}^\top \phi(s, a) - [\mathbb{P}_h \Tilde{\widehat{V}}_{k,h+1}](s, a) \right| \leq \widehat{\beta} \sqrt{\phi(s, a)^\top \Tilde{\Lambda}_{k,h}^{-1} \phi(s, a)},
\]
\[
\left| \Tilde{\Bar{w}}_{k,h}^\top \phi(s, a) - [\mathbb{P}_h \Tilde{\widehat{V}}_{k,h+1}^2](s, a) \right| \leq \Bar{\beta} \sqrt{\phi(s, a)^\top \Tilde{\Lambda}_{k,h}^{-1} \phi(s, a)},
\]
\[
\left| \Tilde{\widecheck{w}}_{k,h}^\top \phi(s, a) - [\mathbb{P}_h \Tilde{\widecheck{V}}_{k,h+1}](s, a) \right| \leq \widecheck{\beta} \sqrt{\phi(s, a)^\top \Tilde{\Lambda}_{k,h}^{-1} \phi(s, a)},
\]
where
\[
\widehat{\beta} = \widecheck{\beta} = O\left(HL\sqrt{d \Tilde{\lambda}_{\Lambda}} + \sqrt{d^3 H^2 \log^2 \left(\frac{HK^{4}L^{2}d}{\delta \Tilde{\lambda}_{\Lambda}}\right)} \right),
\]
and
\[
\bar{\beta} = O\left(H^{2}L^{2} \sqrt{d\Tilde{\lambda}_{\Lambda}} + \sqrt{d^3 H^4 \log^2 \left(\frac{HK^{4}L^{2}d}{\delta \Tilde{\lambda}_{\Lambda}}\right)} \right).
\]

The event \( \mathcal{E} \) holds with probability at least \( 1 - 7\delta \).
\end{lemma}

\begin{proof}[Proof of Lemma~\ref{lemma:confidence-bounds-1}]
For any fixed stage $h \in [H]$ and the optimistic private value function $\Tilde{\widehat{V}}_{k,h+1}$, by Lemma~\ref{lem:variance-linear-mdp}, there exists a vector $w_{k, h+1}$ such that $\mathbb{P}_h \Tilde{\widehat{V}}_{k,h+1}(s, a)$ can be represented as $w_{k, h}^{\top}\phi(s, a)$ with $\|w_{k, h}\|_2 \leq H\sqrt{d}$. Then, we can decompose the estimation error $\norm{\Tilde{\widehat{w}}_{k,h} - w_{k, h}}_{\Tilde{\Lambda}_{h, k}}$ as 

\begin{align*}
   &\norm{\Tilde{\Lambda}_{k, h}^{-1} 
    \left[ \sum_{i=1}^{k-1} \Tilde{\Bar{\sigma}}_{i, h}^{-2} \phi \left( s_{h}^{i}, a_{h}^{i} \right)\Tilde{\widehat{V}}_{k, h+1}(s_{h+1}^{i}) + \phi_{1} \right] - \Tilde{\Lambda}_{k, h}^{-1} 
    \left[ 2 \Tilde{\lambda}_{\Lambda} I_{d} 
    + \sum_{i=1}^{k-1} \Tilde{\Bar{\sigma}}_{i, h}^{-2} \phi \left( s_{h}^{i}, a_{h}^{i} \right) \phi \left( s_{h}^{i}, a_{h}^{i} \right)^{\top} 
    + K_{1} \right]w_{k, h}}_{\Tilde{\Lambda}_{h, k}} \\
    &= \norm{\Tilde{\Lambda}_{k, h}^{-1} 
    \sum_{i=1}^{k-1} \Tilde{\Bar{\sigma}}_{i, h}^{-2} \phi \left( s_{h}^{i}, a_{h}^{i} \right) \left( \Tilde{\widehat{V}}_{k, h+1}(s_{h+1}^{i}) - \mathbb{P}_{h}\Tilde{\widehat{V}}_{k, h+1}(s_{h}^{i}, a_{h}^{i}) \right) + \Tilde{\Lambda}_{k, h}^{-1}\phi_{1} + \Tilde{\Lambda}_{k, h}^{-1}K_{1}w_{k, h} - 2 \Tilde{\lambda}_{\Lambda} \Tilde{\Lambda}_{k, h}^{-1}w_{k, h}}_{\Tilde{\Lambda}_{h, k}} \\
    &\leq \norm{\Tilde{\Lambda}_{k, h}^{-1}\phi_{1}}_{\Tilde{\Lambda}_{h, k}} + \norm{\Tilde{\Lambda}_{k, h}^{-1}w_{k, h}K_{1}}_{\Tilde{\Lambda}_{h, k}} + \norm{2 \Tilde{\lambda}_{\Lambda} \Tilde{\Lambda}_{k, h}^{-1}w_{k, h}}_{\Tilde{\Lambda}_{h, k}} + \\
    &\norm{\Tilde{\Lambda}_{k, h}^{-1} 
    \sum_{i=1}^{k-1} \Tilde{\Bar{\sigma}}_{i, h}^{-2} \phi \left( s_{h}^{i}, a_{h}^{i} \right) \left( \Tilde{\widehat{V}}_{k, h+1}(s_{h+1}^{i}) - \mathbb{P}_{h}\Tilde{\widehat{V}}_{k, h+1}(s_{h}^{i}, a_{h}^{i}) \right)}_{\Tilde{\Lambda}_{h, k}}
\end{align*}

where the first inequality holds from $\norm{a + b}_{\Sigma} \leq \norm{a}_{\Sigma} + \norm{b}_{\Sigma}$. For the first term, we know that by construction, $\Tilde{\Lambda}_{k, h}^{-1} \preceq 1 / \Tilde{\lambda}_{\Lambda}$. Additionally, by utility (Lemma~\ref{lemma:private-utility-analysis}), we have that $\norm{\phi_{1}}_{2} \leq L$. Putting these together, we get
\[
    \norm{\Tilde{\Lambda}_{k, h}^{-1}\phi_{1}}_{\Tilde{\Lambda}_{h, k}} \leq L \sqrt{\frac{1}{\Tilde{\lambda}_{\Lambda}}} \leq HL \sqrt{d \Tilde{\lambda}_{\Lambda}}
\]
For the second term, we have that $\|w_{k, h}\|_2 \leq H\sqrt{d}$. Again, by utility, we have that $\norm{K_{1}}_{2} \leq \Tilde{\lambda}_{\Lambda}$. Thus, we get 
\[
    \norm{\Tilde{\Lambda}_{k, h}^{-1}w_{k, h}K_{1}}_{\Tilde{\Lambda}_{h, k}} \leq H \sqrt{d \Tilde{\lambda}_{\Lambda}} \leq HL \sqrt{d \Tilde{\lambda}_{\Lambda}}
\]
For the third term, using the facts we have described above, we get 
\[
    \norm{2 \Tilde{\lambda}_{\Lambda} \Tilde{\Lambda}_{k, h}^{-1}w_{k, h}}_{\Tilde{\Lambda}_{h, k}} \leq 2H \sqrt{d \Tilde{\lambda}_{\Lambda}} \leq 2HL \sqrt{d \Tilde{\lambda}_{\Lambda}}
\]
Lastly, for the last term, we apply Lemma~\ref{lemma:martingale-ineq-weighted-linear} with the following optimistic value function class $\widehat{\mathcal{V}}_{h}$ and $\varepsilon = H\sqrt{d\Tilde{\lambda}_{\Lambda}} / K$, then for any fixed $h \in [H]$, with probability atleast $1 - \delta / H$, for all episodes $k \in [K]$, we have 
\begin{align*}
    &\norm{\Tilde{\Lambda}_{k, h}^{-1} 
    \sum_{i=1}^{k-1} \Tilde{\Bar{\sigma}}_{i, h}^{-2} \phi \left( s_{h}^{i}, a_{h}^{i} \right) \left( \Tilde{\widehat{V}}_{k, h+1}(s_{h+1}^{i}) - \mathbb{P}_{h}\Tilde{\widehat{V}}_{k, h+1}(s_{h}^{i}, a_{h}^{i}) \right)}_{\Tilde{\Lambda}_{h, k}} \\
    &\leq \sqrt{4C^2H^2 \left[ \frac{d}{2} \log\left(1 + \frac{kC^2}{\Tilde{\lambda}_{\Lambda}}\right) + \log\left(\frac{HN_\varepsilon}{\delta}\right) \right] + \frac{8k^2C^4\varepsilon^2}{\Tilde{\lambda}_{\Lambda}}} \\
    &\leq \sqrt{4H \left[ \frac{d}{2} \log\left(1 + \frac{k}{\Tilde{\lambda}_{\Lambda} H}\right) + \log\left(\frac{HN_\varepsilon}{\delta}\right) \right] + \frac{8k^2\varepsilon^2}{\Tilde{\lambda}_{\Lambda} H^{2}}} \\
    &\leq \sqrt{4H \left[ \frac{d}{2} \log\left(1 + \frac{k}{\Tilde{\lambda}_{\Lambda} H}\right) + \log\left(\frac{HN_\varepsilon}{\delta}\right) \right] + 8} \\
    &= O \left(\sqrt{d^3 H^2 \log^2 \left(\frac{HK^{4}L^{2}d}{\delta \Tilde{\lambda}_{\Lambda}}\right)} \right) \\
\end{align*}
where the first inequality holds due to Lemma~\ref{lemma:martingale-ineq-weighted-linear}, the second inequality holds since $\Tilde{\bar{\sigma}}_{i,h}^{-2} \leq 1 / \sqrt{H}$, and the last inequality holds due to Lemma~\ref{lemma:optimistic-value-function-covering} and $\varepsilon = H\sqrt{d\Tilde{\lambda}_{\Lambda}} / K$. Putting everything together, we get
\[
    \norm{\Tilde{\widehat{w}}_{k,h} - w_{k, h}}_{\Tilde{\Lambda}_{h, k}} \leq O\left(HL\sqrt{d \Tilde{\lambda}_{\Lambda}} + \sqrt{d^3 H^2 \log^2 \left(\frac{HK^{4}L^{2}d}{\delta \Tilde{\lambda}_{\Lambda}}\right)} \right) = \widehat{\beta}
\]
Thus, using this, we can say
\begin{align*}
    \left| \Tilde{\widehat{w}}_{k,h}^\top \phi(s, a) - [\mathbb{P}_h \Tilde{\widehat{V}}_{k,h+1}](s, a) \right| &=  \left| \Tilde{\widehat{w}}_{k,h}^\top \phi(s, a) - w_{k, h}^{\top}  \phi(s, a)  \right| \\
    &\leq  \norm{\Tilde{\widehat{w}}_{k,h} - w_{k, h}}_{\Tilde{\Lambda}_{h, k}} \norm{\phi(s,a)}_{\Tilde{\Lambda}_{h, k}} \\
    &\leq \widehat{\beta} \sqrt{\phi(s, a)^\top \Tilde{\Lambda}_{k,h}^{-1} \phi(s, a)}
\end{align*}
where the first inequality holds due to Cauchy-Schwartz inequality. Replacing the value function class by the pessimistic value function class $\widecheck{\mathcal{V}}$ or the squared value function class $\widehat{\mathcal{V}}^{2}$ and using the same exact proof as above, we can derive the other upper estimation errors
\[
\left| \Tilde{\Bar{w}}_{k,h}^\top \phi(s, a) - [\mathbb{P}_h \Tilde{\widehat{V}}_{k,h+1}^2](s, a) \right| \leq \Bar{\beta} \sqrt{\phi(s, a)^\top \Tilde{\Lambda}_{k,h}^{-1} \phi(s, a)},
\]
\[
\left| \Tilde{\widecheck{w}}_{k,h}^\top \phi(s, a) - [\mathbb{P}_h \Tilde{\widecheck{V}}_{k,h+1}](s, a) \right| \leq \widecheck{\beta} \sqrt{\phi(s, a)^\top \Tilde{\Lambda}_{k,h}^{-1} \phi(s, a)},
\]
where 
\[
\widecheck{\beta} = O\left(HL\sqrt{d \Tilde{\lambda}_{\Lambda}} + \sqrt{d^3 H^2 \log^2 \left(\frac{HK^{4}L^{2}d}{\delta \Tilde{\lambda}_{\Lambda}}\right)} \right),
\]
and
\[
\bar{\beta} = O\left(H^{2}L^{2} \sqrt{d\Tilde{\lambda}_{\Lambda}} + \sqrt{d^3 H^4 \log^2 \left(\frac{HK^{4}L^{2}d}{\delta \Tilde{\lambda}_{\Lambda}}\right)} \right).
\]
\end{proof}

Now, we provide a bound on the variance estimator.

\begin{lemma}
\label{lemma:confidence-bounds-2}
Let $\widetilde{\mathcal{E}}_h$ be the event such that for all episodes $k \in [K]$, stages $h \leq h' \leq H$, and state-action pairs $(s, a) \in \mathcal{S} \times \mathcal{A}$, the weight vector $\widehat{w}_{k,h}$ satisfies 
\[
\left| \Tilde{\widehat{w}}_{k,h^{\prime}}^\top \phi(s, a) - [\mathbb{P}_h \Tilde{\widehat{V}}_{k,h^{\prime}+1}](s, a) \right| \leq \beta \sqrt{\phi(s, a)^\top \Tilde{\Lambda}_{k,h^{\prime}}^{-1} \phi(s, a)} \tag{B.1} \label{eq:bernstein-bound}
\]
where 
\[
\beta = O\left(HL\sqrt{d\Tilde{\lambda}_{\Lambda}} + \sqrt{d \log^2\left(1 + \left(\frac{HK^{4}L^{2}d}{\delta \Tilde{\lambda}_{\Lambda}} \right) \right)} \right).
\]
On the event $\mathcal{E}$ and $\widetilde{\mathcal{E}}_{h+1}$, for each episode $k \in [K]$ and stage $h$, the estimated variance satisfies:
\[
\left|[\overline{\mathbb{V}}_{h}\Tilde{\widehat{V}}_{k, h+1}](s_k^h, a_k^h) - [\mathbb{V}_h \Tilde{\widehat{V}}_{k, h+1}](s_k^h, a_k^h) \right| \leq E_{k,h},
\]
and
\[
\left|[\overline{\mathbb{V}}_{h}\Tilde{\widehat{V}}_{k, h+1}](s_h^k, a_h^k) - [\mathbb{V}_h V_h^{*}](s_h^k, a_h^k)\right| \leq E_{k,h} + D_{k,h}.
\]
where
\[
    E_{k, h} = \mathrm{min} \left\{ \overline{\beta}_{k} \norm{\Tilde{\Lambda}_{k,h}^{-1/2}\phi(s_{h}^{k}, a_{h}^{k})}_{2} ,H^{2} \right\} + \mathrm{min} \left\{ 2H\widehat{\beta}_{k} \norm{ \Tilde{\Lambda}_{k,h}^{-1/2}\phi(s_{h}^{k}, a_{h}^{k})}_{2} ,H^{2} \right\}
\]
and
\[
    D_{k, h} = \mathrm{min} \left\{ 4d^{3}H^{2} \left( \Tilde{\widehat{w}}_{k, h}^{\top} \phi(s, a) - \Tilde{\widecheck{w}}_{k, h}^{\top} \phi(s, a) + 2\widehat{\beta}_{k} \sqrt{\phi(s, a)^\top \Tilde{\Lambda}_{k,h}^{-1} \phi(s, a)} \right), d^{3}H^{3} \right\}
\]
\end{lemma}

\begin{proof}[Proof of Lemma~\ref{lemma:confidence-bounds-2}]
    We will first use Lemma~\ref{lemma:confidence-bounds-1}
    \begin{align*}
        &\left|[\overline{\mathbb{V}}_{ h}\Tilde{\widehat{V}}_{k, h+1}](s_k^h, a_k^h) - [\mathbb{V}_h \Tilde{\widehat{V}}_{k, h+1}](s_k^h, a_k^h) \right| \\
        &= \left| [ \Tilde{\overline{w}}_{k, h} \phi(s_{h}^{k}, a_{h}^{k}) ]_{[0, H^{2}]} - [\Tilde{\widehat{w}}_{k, h} \phi(s_{h}^{k}, a_{h}^{k}) ]_{[0, H]}^{2} - [\mathbb{P}_{h}\Tilde{\widehat{V}}_{k, h+1}]^{2}](s_{h}^{k}, a_{h}^{k}) -\left( [\mathbb{P}_{h}\Tilde{\widehat{V}}_{k, h+1}]](s_{h}^{k}, a_{h}^{k}) \right)^{2}  \right| \\
        &\leq \left| [ \Tilde{\overline{w}}_{k, h} \phi(s_{h}^{k}, a_{h}^{k}) ]_{[0, H^{2}]} -  [\mathbb{P}_{h}\Tilde{\widehat{V}}_{k, h+1}]^{2}](s_{h}^{k}, a_{h}^{k})\right| + \left| [\Tilde{\widehat{w}}_{k, h} \phi(s_{h}^{k}, a_{h}^{k}) ]_{[0, H]}^{2} - \left( [\mathbb{P}_{h}\Tilde{\widehat{V}}_{k, h+1}]_{k, h+1}](s_{h}^{k}, a_{h}^{k}) \right)^{2} \right| \\
        &= \left| [ \Tilde{\overline{w}}_{k, h} \phi(s_{h}^{k}, a_{h}^{k}) ]_{[0, H^{2}]} -  [\mathbb{P}_{h}\Tilde{\widehat{V}}_{k, h+1}]_{k, h+1}^{2}](s_{h}^{k}, a_{h}^{k})\right| \\
        &+ \left| [\Tilde{\widehat{w}}_{k, h} \phi(s_{h}^{k}, a_{h}^{k}) ]_{[0, H]} + [\mathbb{P}_{h}\Tilde{\widehat{V}}_{k, h+1}]_{k, h+1}](s_{h}^{k}, a_{h}^{k})  \right| \left| [\Tilde{\widehat{w}}_{k, h} \phi(s_{h}^{k}, a_{h}^{k}) ]_{[0, H]} - [\mathbb{P}_{h}\Tilde{\widehat{V}}_{k, h+1}]_{k, h+1}](s_{h}^{k}, a_{h}^{k})  \right| \\
        &\leq \mathrm{min} \left\{ \overline{\beta}_{k} \norm{\Tilde{\Lambda}_{k,h}^{-1/2}\phi(s_{h}^{k}, a_{h}^{k})}_{2} ,H^{2} \right\} + \mathrm{min} \left\{ 2H\widehat{\beta}_{k} \norm{ \Tilde{\Lambda}_{k,h}^{-1/2}\phi(s_{h}^{k}, a_{h}^{k})}_{2} ,H^{2} \right\} \\
        &= E_{k, h}
    \end{align*}
    where the first inequality holds from the triangle inequality and the second inequality holds from conditioning on $\mathcal{E}$ and the fact that $0 \leq [\Tilde{\widehat{w}}_{k, h} \phi_{s_{h}^{k}, s_{h}^{k}} ]_{[0, H]} + [\mathbb{P}_{h}V_{k, h+1}](s_{h}^{k}, a_{h}^{k}) \leq 2H$. Now,
    \begin{align*}
        &\left|[\mathbb{V}_{h}\Tilde{\widehat{V}}_{k, h+1}](s_h^k, a_h^k) - [\mathbb{V}_h V_h^{*}](s_h^k, a_h^k)\right| \\
        &= \left| [\mathbb{P}_{h}\Tilde{\widehat{V}}_{k, h+1}^{2}](s_{h}^{k}, a_{h}^{k}) - \left( [\mathbb{P}_{h}\Tilde{\widehat{V}}_{k, h+1}](s_{h}^{k}, a_{h}^{k}) \right)^{2} - [\mathbb{P}_{h}V_{h+1}^{* 2}](s_{h}^{k}, a_{h}^{k}) + \left( [\mathbb{P}_{h}V_{h+1}^{*}](s_{h}^{k}, a_{h}^{k}) \right)^{2}  \right| \\
        &\leq \left| [\mathbb{P}_{h} \left( \Tilde{\widehat{V}}_{k, h+1} - V_{h+1}^{*} \right)  \left( \Tilde{\widehat{V}}_{k, h+1} + V_{h+1}^{*} \right)](s_{h}^{k}, a_{h}^{k}) \right| \\
        &+ \left| \left( [\mathbb{P}_{h}\Tilde{\widehat{V}}_{h+1}](s_{h}^{k}, a_{h}^{k}) - [\mathbb{P}_{h}V^{*}_{h+1}](s_{h}^{k}, a_{h}^{k}) \right) \left( [\mathbb{P}_{h}\Tilde{\widehat{V}}_{k, h+1}](s_{h}^{k}, a_{h}^{k}) + [\mathbb{P}_{h}V^{*}_{h+1}](s_{h}^{k}, a_{h}^{k}) \right) \right| \\
        &\leq 4H \left( [\mathbb{P}_{h}\Tilde{\widehat{V}}_{k, h+1}](s_{h}^{k}, a_{h}^{k}) - [\mathbb{P}_{h}V^{*}_{h+1}](s_{h}^{k}, a_{h}^{k}) \right)
    \end{align*}
    where the first inequality holds from triangle inequality and the second inequality holds due to Lemma~\ref{lemma:optimism-and-pessimism} and the fact that $0 \leq V_{h+1}^{*}(s^{\prime} \leq V_{k, h+1}(s^{\prime}) \leq H$. Now, if we condition on $\mathcal{E}$ and $\Tilde{\mathcal{E}}$, we get 
    \begin{align*}
        &\left( [\mathbb{P}_{h}\Tilde{\widehat{V}}_{k, h+1}](s_{h}^{k}, a_{h}^{k}) - [\mathbb{P}_{h}V^{*}_{h+1}](s_{h}^{k}, a_{h}^{k}) \right) \\
        &\leq \left( [\mathbb{P}_{h}\Tilde{\widehat{V}}_{k, h+1}](s_{h}^{k}, a_{h}^{k}) - [\mathbb{P}_{h}\widecheck{V}_{k, h+1}](s_{h}^{k}, a_{h}^{k}) \right) \\
        &\leq \Tilde{\widehat{w}}_{k, h}^{\top} \phi(s, a) + \widehat{\beta} \sqrt{\phi(s, a)^\top \Tilde{\Lambda}_{k,h}^{-1} \phi(s, a)} - \Tilde{\widecheck{w}}_{k, h}^{\top} \phi(s, a) + \widecheck{\beta} \sqrt{\phi(s, a)^\top \Tilde{\Lambda}_{k,h}^{-1} \phi(s, a)}
    \end{align*}
    where the first inequality holds due to Lemma~\ref{lemma:optimism-and-pessimism} and the last inequality holds by Lemma~\ref{lemma:confidence-bounds-1}. Combining results, we get 
    \begin{align*}
        &\left| [\overline{\mathbb{V}}_{h}\Tilde{\widehat{V}}_{k, h+1}] (s_{h}^{k}, a_{h}^{k})  - [\mathbb{V}_{h}V_{h+1}^{*}] \right| \\
        &\leq \left| [\overline{\mathbb{V}}_{h}\Tilde{\widehat{V}}_{k, h+1}] (s_{h}^{k}, a_{h}^{k})  - [\mathbb{V}_{h}\Tilde{\widehat{V}}_{k, h+1}^{*}](s_{h}^{k}, a_{h}^{k}) \right| + \left| [\mathbb{V}_{h}\Tilde{\widehat{V}}_{k, h+1}] (s_{h}^{k}, a_{h}^{k})  - [\mathbb{V}_{h}V_{h+1}^{*}](s_{h}^{k}, a_{h}^{k}) \right| \\
        &\leq E_{k, h} + \mathrm{min} \left\{ 4H \left( \Tilde{\widehat{w}}_{k, h}^{\top} \phi(s, a) + \widehat{\beta} \sqrt{\phi(s, a)^\top \Tilde{\Lambda}_{k,h}^{-1} \phi(s, a)} - \Tilde{\widecheck{w}}_{k, h}^{\top} \phi(s, a) + \widecheck{\beta} \sqrt{\phi(s, a)^\top \Tilde{\Lambda}_{k,h}^{-1} \phi(s, a)} \right), H^{2} \right\}
    \end{align*}
\end{proof}

We also have another upper bound on the variance estimator.

\begin{lemma}
    \label{lemma:upper-confidence-bound3}
    On the event $\mathcal{E}$ and $\tilde{\mathcal{E}}_{h+1}$, for any episode $k$ and $i > k$, we have
    \[
    [\mathbb{V}_h(\Tilde{\widehat{V}}_{i, h+1} - V_{h+1}^*)](s_h^k, a_h^k) \leq D_{k, h}/d^3 H.
    \]
\end{lemma}

\begin{proof}[Proof of Lemma~\ref{lemma:upper-confidence-bound3}]
   On the event $\mathcal{E}$ and $\tilde{\mathcal{E}}_{h+1}$, we have
   \begin{align*}
       [\mathbb{V}_h(\Tilde{\widehat{V}}_{i, h+1} - \mathbb{V}_{h+1}^*)](s_h^k, a_h^k) &\leq [\mathbb{P}_h(\Tilde{\widehat{V}}_{i, h+1} - V_{h+1}^*)^2](s_h^k, a_h^k) \\
       &\leq 2H[\mathbb{P}_h(\Tilde{\widehat{V}}_{i, h+1} - V_{h+1}^*)](s_h^k, a_h^k) \\
       &\leq 2H \left( [\mathbb{P}_h(\Tilde{\widehat{V}}_{i, h+1}](s_{h}^{k}, a_{h}^{k}) - [\mathbb{P}_h(\Tilde{\widecheck{V}}_{i, h+1}](s_{h}^{k}, a_{h}^{k})   \right) \\
       &\leq 2H \left( [\mathbb{P}_h(\Tilde{\widehat{V}}_{k, h+1}](s_{h}^{k}, a_{h}^{k}) - [\mathbb{P}_h(\Tilde{\widecheck{V}}_{i, h+1}](s_{h}^{k}, a_{h}^{k})   \right) \\
       &\leq 2H \left( \Tilde{\widehat{w}}_{k, h}^{\top} \phi(s, a) + \widehat{\beta} \sqrt{\phi(s, a)^\top \Tilde{\Lambda}_{k,h}^{-1} \phi(s, a)} - \Tilde{\widecheck{w}}_{k, h}^{\top} \phi(s, a) + \widecheck{\beta} \sqrt{\phi(s, a)^\top \Tilde{\Lambda}_{k,h}^{-1} \phi(s, a)}  \right)
   \end{align*}
where the first inequality holds due to $\text{Var}(x) \leq \mathbb{E}[x^2]$, the second and third inequalities hold due to Lemma~\ref{lemma:optimism-and-pessimism} with the fact that $0 \leq \Tilde{\widehat{V}}_{i, h+1}(s'), V_{h+1}^*(s') \leq H$, the fourth inequality holds because $V_{k, h+1} \geq V_{i, h+1}$ from the update rule in Algorithm~\ref{alg:dp-lsvi_ucb++}, and the fifth inequality holds due to Lemma~\ref{lemma:confidence-bounds-1}. On the other hand, since the value functions satisfy $0 \leq \Tilde{\widehat{V}}_{i, h+1}(s'), V_{h+1}^*(s') \leq H$, we have
\[
[V_h(V_{i, h+1} - V_{h+1}^*)](s_h^k, a_h^k) \leq \frac{d^3 H^3}{d^3 H} = H^{2}.
\]
\end{proof}

Here, we prove the optimism and pessimism of our privatized value function which we will use for the regret analysis.

\begin{lemma}[Privatized Optimism and Pessimism]
    \label{lemma:optimism-and-pessimism}
    On the event $\mathcal{E}$ and $\tilde{\mathcal{E}}_h$, for all episodes $k \in [K]$ and stages $h \leq h' \leq H$, we have
\[
\Tilde{\widehat{Q}}_{k, h}(s, a) \geq Q_h^*(s, a) \geq \Tilde{\widecheck{Q}}_{k, h}(s, a).
\]
In addition, we have
\[
\Tilde{\widehat{V}}_{k, h}(s) \geq V_h^*(s) \geq \Tilde{\widecheck{V}}_{k, h}(s).
\]
\end{lemma}

\begin{proof}[Proof of Lemma~\ref{lemma:optimism-and-pessimism}]
    As we would usually do, we will prove optimism and pessimism using induction. First, consider the base case $H+1$. For all states $s \in S$ and actions $a \in A$, we have
    
\[
\Tilde{\widehat{Q}}_{k, H+1}(s, a) = Q_h^*(s, a) = \Tilde{\widecheck{Q}}_{k, h}(s, a) = 0 \quad \text{and} \quad \Tilde{\widehat{V}}_{k, h}(s) \geq V_h^*(s) \geq \Tilde{\widecheck{V}}_{k, h}(s) = 0.
\]

Thus, we have shown the base case. Now, consider stage $h + 1$. Since the event $\tilde{\mathcal{E}}_h$ directly implies the event $\tilde{\mathcal{E}}_{h+1}$, according to the induction hypothesis, we have

\[
\Tilde{\widehat{V}}_{k, h+1}(s) \geq V_{h+1}^*(s) \geq \Tilde{\widecheck{V}}_{k, h+1}(s).
\]

Thus, for all episodes $k \in [K]$, we have

\[
r_h(s, a) + \Tilde{\widehat{w}}_k^\top \phi(s, a) + \widehat{\beta} \sqrt{\phi(s, a)^\top \Tilde{\Lambda}_{k, h}^{-1} \phi(s, a)} - Q_h^*(s, a) \geq [\mathbb{P}_h(\Tilde{\widehat{V}}_{k, h+1} - V_{h+1}^*)](s, a) \geq 0,
\]

where the first inequality holds by conditioning on event $\tilde{\mathcal{E}}_h$. Additionally, we have

\[
Q_h^*(s, a) \leq \min \left\{ \min_{1 \leq i \leq k} \big( r_h(s, a) + \Tilde{\widehat{w}}_i^\top \phi(s, a) + \widehat{\beta} \sqrt{\phi(s, a)^\top \Tilde{\Lambda}_{i, h}^{-1} \phi(s, a)} \big), H \right\} \leq \Tilde{\widehat{Q}}_{k, h}(s, a).
\]

With a similar argument, for the pessimistic action-value function $\Tilde{\widecheck{Q}}_{k, h}(s, a)$, we have
\[
r_h(s, a) + \Tilde{\widecheck{w}}_k^\top \phi(s, a) - \widecheck{\beta} \sqrt{\phi(s, a)^\top \Tilde{\Lambda}_{k, h}^{-1} \phi(s, a)} - Q_h^*(s, a) \leq [\mathbb{P}_h(\Tilde{\widecheck{V}}_{k, h+1} - V_{h+1}^*)](s, a) \leq 0.
\]
Since the optimal value function is lower bounded by $Q_h^*(s, a) \geq 0$, the result further implies that
\[
Q_h^*(s, a) \geq \max \left\{ \max_{1 \leq i \leq k} \big( r_h(s, a) + \Tilde{\widecheck{w}}_{\text{last}, h}^\top \phi(s, a) + \widecheck{\beta} \sqrt{\phi(s, a)^\top \Tilde{\Lambda}_{\text{last}, h}^{-1} \phi(s, a)} \big), 0 \right\} \geq \Tilde{\widecheck{Q}}_{k, h}(s, a).
\]

In addition, we have
\[
\Tilde{\widehat{V}}_{k, h}(s) = \max_a \Tilde{\widehat{Q}}_{i, h}(s, a) \geq \min_{1 \leq i \leq k} \max_a Q_h^*(s, a) = V_h^*(s),
\]
\[
\Tilde{\widecheck{V}}_{k, h}(s) = \max_a \Tilde{\widehat{Q}}_{i, h}(s, a) \leq \max_{1 \leq i \leq k} \max_a Q_h^*(s, a) = V_h^*(s),
\]
\end{proof}
Now, we will also provide a Bernstein-type upper bound on the estimation error using what have proven so far. This is much sharper than Lemma~\ref{lemma:confidence-bounds-1}.

\begin{lemma}
    Define $\widetilde{\mathcal{E}} = \widetilde{\mathcal{E}}_1$ as the event such that \ref{eq:bernstein-bound} holds for all stages $h \in [H]$. On the events $\mathcal{E}$, event $\Tilde{\mathcal{E}}$ holds with probability at least $1 - \delta$
\end{lemma}

\begin{proof}
    For any fixed stage $h \in [H]$ and the optimistic private value function $\Tilde{\widehat{V}}_{k,h+1}$, by Lemma~\ref{lem:variance-linear-mdp}, there exists a vector $w_{k, h+1}$ such that $\mathbb{P}_h \Tilde{\widehat{V}}_{k,h+1}(s, a)$ can be represented as $w_{k, h}^{\top}\phi(s, a)$ with $\|w_{k, h}\|_2 \leq H\sqrt{d}$. Then, we can decompose the estimation error $\norm{\Tilde{\widehat{w}}_{k,h} - w_{k, h}}_{\Tilde{\Lambda}_{h, k}}$ as 

\begin{align*}
   &\norm{\Tilde{\Lambda}_{k, h}^{-1} 
    \left[ \sum_{i=1}^{k-1} \Tilde{\Bar{\sigma}}_{i, h}^{-2} \phi \left( s_{h}^{i}, a_{h}^{i} \right)\Tilde{\widehat{V}}_{k, h+1}(s_{h+1}^{i}) + \phi_{1} \right] - \Tilde{\Lambda}_{k, h}^{-1} 
    \left[ 2 \Tilde{\lambda}_{\Lambda} I_{d} 
    + \sum_{i=1}^{k-1} \Tilde{\Bar{\sigma}}_{i, h}^{-2} \phi \left( s_{h}^{i}, a_{h}^{i} \right) \phi \left( s_{h}^{i}, a_{h}^{i} \right)^{\top} 
    + K_{1} \right]w_{k, h}}_{\Tilde{\Lambda}_{h, k}} \\
    &= \norm{\Tilde{\Lambda}_{k, h}^{-1} 
    \sum_{i=1}^{k-1} \Tilde{\Bar{\sigma}}_{i, h}^{-2} \phi \left( s_{h}^{i}, a_{h}^{i} \right) \left( \Tilde{\widehat{V}}_{k, h+1}(s_{h+1}^{i}) - \mathbb{P}_{h}\Tilde{\widehat{V}}_{k, h+1}(s_{h}^{i}, a_{h}^{i}) \right) + \Tilde{\Lambda}_{k, h}^{-1}\phi_{1} + \Tilde{\Lambda}_{h, k}^{-1}w_{k, h}K_{1} - 2 \Tilde{\lambda}_{\Lambda} \Tilde{\Lambda}_{k, h}^{-1}w_{k, h}}_{\Tilde{\Lambda}_{h, k}} \\
    &\leq \norm{\Tilde{\Lambda}_{k, h}^{-1}\phi_{1}}_{\Tilde{\Lambda}_{h, k}} + \norm{\Tilde{\Lambda}_{h, k}^{-1} w_{k, h}K_{1}}_{\Tilde{\Lambda}_{h, k}} + \norm{2 \Tilde{\lambda}_{\Lambda} \Tilde{\Lambda}_{k, h}^{-1}w_{k, h}}_{\Tilde{\Lambda}_{h, k}} + \\
    &\norm{\Tilde{\Lambda}_{k, h}^{-1} 
    \sum_{i=1}^{k-1} \Tilde{\Bar{\sigma}}_{i, h}^{-2} \phi \left( s_{h}^{i}, a_{h}^{i} \right) \left( \Tilde{\widehat{V}}_{k, h+1}(s_{h+1}^{i}) - \mathbb{P}_{h}\Tilde{\widehat{V}}_{k, h+1}(s_{h}^{i}, a_{h}^{i}) \right)}_{\Tilde{\Lambda}_{h, k}}
\end{align*}

where the first inequality holds from $\norm{a + b}_{\Sigma} \leq \norm{a}_{\Sigma} + \norm{b}_{\Sigma}$. For the first term, we know that by construction, $\Tilde{\Lambda}_{k, h}^{-1} \preceq 1 / \Tilde{\lambda}_{\Lambda}$. Additionally, by utility (Lemma~\ref{lemma:private-utility-analysis}), we have that $\norm{\phi_{1}}_{2} \leq L$. Putting these together, we get
\[
    \norm{\Tilde{\Lambda}_{k, h}^{-1}\phi_{1}}_{\Tilde{\Lambda}_{h, k}} \leq L \sqrt{\frac{1}{\Tilde{\lambda}_{\Lambda}}} \leq HL \sqrt{d \Tilde{\lambda}_{\Lambda}}
\]
For the second term, we have that $\|w_{k, h}\|_2 \leq H\sqrt{d}$. Again, by utility, we have that $\norm{K_{1}}_{2} \leq \Tilde{\lambda}_{\Lambda}$. Thus, we get 
\[
    \norm{\Tilde{\Lambda}_{h, k}^{-1}w_{k, h}K_{1}}_{\Tilde{\Lambda}_{h, k}} \leq H \sqrt{d \Tilde{\lambda}_{\Lambda}} \leq HL \sqrt{d \Tilde{\lambda}_{\Lambda}}
\]
For the third term, using the facts we have described above, we get 
\[
    \norm{2 \Tilde{\lambda}_{\Lambda} \Tilde{\Lambda}_{k, h}^{-1}w_{k, h}}_{\Tilde{\Lambda}_{h, k}} \leq 2H \sqrt{d \Tilde{\lambda}_{\Lambda}} \leq 2HL \sqrt{d \Tilde{\lambda}_{\Lambda}}
\]

For the last term,

\begin{align*}
    &\norm{\Tilde{\Lambda}_{k, h}^{-1} 
    \sum_{i=1}^{k-1} \Tilde{\Bar{\sigma}}_{i, h}^{-2} \phi \left( s_{h}^{i}, a_{h}^{i} \right) \left( \Tilde{\widehat{V}}_{k, h+1}(s_{h+1}^{i}) - \mathbb{P}_{h}\Tilde{\widehat{V}}_{k, h+1}(s_{h}^{i}, a_{h}^{i}) \right)}_{\Tilde{\Lambda}_{h, k}} \\
    &= \norm{ 
    \sum_{i=1}^{k-1} \Tilde{\Bar{\sigma}}_{i, h}^{-2} \phi \left( s_{h}^{i}, a_{h}^{i} \right) \left( \Tilde{\widehat{V}}_{k, h+1}(s_{h+1}^{i}) - \mathbb{P}_{h}\Tilde{\widehat{V}}_{k, h+1}(s_{h}^{i}, a_{h}^{i}) \right)}_{\Tilde{\Lambda}_{h, k}^{-1}} \\
    &\leq \norm{ 
    \sum_{i=1}^{k-1} \Tilde{\Bar{\sigma}}_{i, h}^{-2} \phi \left( s_{h}^{i}, a_{h}^{i} \right) \left( V^{*}_{h+1}(s_{h+1}^{i}) - \mathbb{P}_{h}V^{*}_{h+1}(s_{h}^{i}, a_{h}^{i}) \right)}_{\Tilde{\Lambda}_{h, k}^{-1}} \\
    &+ \norm{ 
    \sum_{i=1}^{k-1} \Tilde{\Bar{\sigma}}_{i, h}^{-2} \phi \left( s_{h}^{i}, a_{h}^{i} \right) \left( \Delta \Tilde{\widehat{V}}_{k, h+1}(s_{h+1}^{i}) - \mathbb{P}_{h}\Delta \Tilde{\widehat{V}}_{k, h+1}(s_{h}^{i}, a_{h}^{i}) \right)}_{\Tilde{\Lambda}_{h, k}^{-1}}
\end{align*}
where we define $\Delta \Tilde{\widehat{V}}_{k, h+1} = \Tilde{\widehat{V}}_{k, h+1} - V_{h+1}^{*}$. For the first term, we use the result from Zhou and Gu (Lemma~\ref{aux:zhou-and-gu}) where 

\[
    x_{i} = \Tilde{\overline{\sigma}}_{i, h}^{-1}\phi(s_{h}^{i}, a_{h}^{i})
\]

and

\[
    \eta_{i} = \mathbbm{1} \left\{ [\mathbb{V}_{h} V_{h+1}^{*}](s_{h}^{i}, a_{h}^{i}) \leq \Tilde{\overline{\sigma}}_{i, h}^{2} \right\} \left( \Tilde{\overline{\sigma}}_{i, h}^{-1} (V^*_{h+1}(s^i_{h+1}) - [\mathbb{P}_hV^*_{h+1}](s^i_h, a^i_h)) \right)
\]
Then, we have the following:
\[
\|x_i\|_2 = \left\|\Tilde{\overline{\sigma}}_{i, h}^{-1}\phi(s_{h}^{i}, a_{h}^{i})\right\|_2 \leq \frac{\|\phi(s^i_h, a^i_h)\|_2}{\sqrt{H}} \leq \frac{1}{\sqrt{H}},
\]
\[
\mathbb{E}[\eta_i | \mathcal{F}_i] = 0, \quad |\eta_t| \leq \left| \left( \Tilde{\overline{\sigma}}_{i, h}^{-1} (V^*_{h+1}(s^i_{h+1}) - [\mathbb{P}_hV^*_{h+1}](s^i_h, a^i_h)) \right) \right| \leq \sqrt{H},
\]
\[
\mathbb{E}[\eta_i^2 | \mathcal{F}_i] = \mathbb{E}\left[
\mathbbm{1} \left\{ [\mathbb{V}_{h} V_{h+1}^{*}](s_{h}^{i}, a_{h}^{i}) \leq \Tilde{\overline{\sigma}}_{i, h}^{2} \right\} \cdot \Tilde{\overline{\sigma}}^{-2}_{i,h} [\mathbb{V}_hV^*_{h+1}](s^i_h, a^i_h)
\right] \leq 1,
\]
\[
\max_i \left\{|\eta_i| \cdot \min\{1, \|x_i\|_{\Tilde{\Lambda}^{-1}_{i,h}}\}\right\} \leq 2H \Tilde{\overline{\sigma}}^{-1}_{i,h} \|x_i\|_{\Tilde{\Lambda}_{i,h}^{-1}} \leq \sqrt{d}.
\]

Thus, with probability at least $1 - \delta/H$, for all $k \in [K]$, we have
\[
\left\|\sum_{i=1}^{k-1} x_i \eta_i\right\|_{\Tilde{\Lambda}^{-1}_{k,h}} \leq O\left(\sqrt{d \log^2\left(1 + \left(\frac{HK^{4}L^{2}d}{\delta \Tilde{\lambda}_{\Lambda}} \right) \right)} \right).
\]

In addition, on the event $\tilde{\mathcal{E}}_{h+1}$ and $\mathcal{E}$, according to Lemma B.2, we have
\[
\Tilde{\overline{\sigma}}^2_{k,h} \geq [\overline{\mathbb{V}}_{k,h} \Tilde{\widehat{V}}_{k,h+1}](s^k_h, a^k_h) + E_{k,h} + D_{k,h} \geq [\mathbb{V}_hV^*_{h+1}](s^k_h, a^k_h),
\]
which further implies that
\[
\norm{ 
    \sum_{i=1}^{k-1} \Tilde{\Bar{\sigma}}_{i, h}^{-2} \phi \left( s_{h}^{i}, a_{h}^{i} \right) \left( V^{*}_{h+1}(s_{h+1}^{i}) - \mathbb{P}_{h}V^{*}_{h+1}(s_{h}^{i}, a_{h}^{i}) \right)}_{\Tilde{\Lambda}_{h, k}^{-1}}
= \left\|\sum_{i=1}^{k-1} x_i \eta_i\right\|_{\Tilde{\Lambda}^{-1}_{k,h}} \leq O\left(\sqrt{d \log^2\left(1 + \left(\frac{HK^{4}L^{2}d}{\delta \Tilde{\lambda}_{\Lambda}} \right) \right)} \right).
\]

For the second term, we cannot directly use Lemma~\ref{aux:zhou-and-gu} since the stochastic noise 
\[
\Delta \Tilde{\widehat{V}}_{k,h+1}(s^i_{h+1}) - [\mathbb{P}_h(\Delta \Tilde{\widehat{V}}_{k,h+1})](s^i_h, a^i_h)
\]
is not $\mathcal{F}_{i+1}$ measurable. Thus, we need to use the $\varepsilon$-net covering argument. For each episode $i$, the value function $V_{i,h}$ belongs to the optimistic value function class $\mathcal{V}$. If we set $\varepsilon = \sqrt{\Tilde{\lambda}_{\Lambda}}/(4H^2d^2K)$, then according to Lemma~\ref{lemma:optimistic-value-function-covering}, the covering entropy for the function class $\mathcal{V} - V^*_{h+1}$ is upper bounded by
\[
\log N_\varepsilon \leq O(d^3H^2 \log^2(HK^{4}L^{2}d/\Tilde{\lambda}_{\Lambda})).
\]
Then for function $\Tilde{\widehat{V}}_{k,h}$, there must exist a function $\tilde{V}$ in the $\varepsilon$-net, such that
\[
\text{dist}(\Delta \Tilde{\widehat{V}}_{k,h}, \tilde{V}) \leq \varepsilon.
\]
Therefore, the variance of function $\tilde{V}$ is upper bounded by
\[
\begin{aligned}
    [\mathbb{V}_h \tilde{V}](s^i_h, a^i_h) - [\mathbb{V}_h(\Delta \Tilde{\widehat{V}}_{k,h+1})](s^i_h, a^i_h) 
    &= [\mathbb{P}_h \tilde{V}^2](s^i_h, a^i_h) - [\mathbb{P}_h(\Delta \Tilde{\widehat{V}}_{k,h+1})^2](s^i_h, a^i_h) \\
    &\quad + ([\mathbb{P}_h(\Delta \Tilde{\widehat{V}}_{k,h+1})](s^i_h, a^i_h))^2 - (\mathbb{P}_h \tilde{V}(s^i_h, a^i_h))^2 \\
    &\leq 2 \, \text{dist}(\Delta \Tilde{\widehat{V}}_{k,h}, \tilde{V}) \cdot \max_{s'} |\Delta \Tilde{\widehat{V}}_{k,h+1} + \tilde{V}|(s') \\
    &\leq 4H \cdot \text{dist}(\Delta \Tilde{\widehat{V}}_{k,h}, \tilde{V}) \\
    &\leq \frac{1}{d^2}
\end{aligned}
\]
where the first inequality holds due to the definition of distance between different functions, the third inequality holds since $|\Delta V_{k,h+1}(s') + \tilde{V}(s')| \leq 2H$, and the last inequality holds due to the definition of $\varepsilon$-net. Again, we make use of Lemma~\ref{aux:zhou-and-gu} with the following:
\[
x_{i} = \Tilde{\overline{\sigma}}_{i, h}^{-1}\phi(s_{h}^{i}, a_{h}^{i})
\]

and

\[
\eta_i = 1\{[\mathbb{V}_h \tilde{V}](s^i_h, a^i_h) \leq \Tilde{\overline{\sigma}}^2_i / (d^3 H)\} \cdot \Tilde{\overline{\sigma}}_i^{-1} (\tilde{V}(s^i_{h+1}) - [\mathbb{P}_h \tilde{V}](s^i_h, a^i_h)).
\]

Therefore, for $x_t, \eta_t$, we have the following property:
\[
\|x_i\|_2 = \|\Tilde{\overline{\sigma}}_{i, h}^{-1}\phi(s_{h}^{i}, a_{h}^{i})\|_2 \leq \|\phi(s^i_h, a^i_h)\|_2 / \sqrt{H} \leq 1 / \sqrt{H},
\]
\[
\mathbb{E}[\eta_i \mid \mathcal{F}_i] = 0, \quad |\eta_t| \leq \left|\Tilde{\overline{\sigma}}_i^{-1} (V^*_h(s^i_{h+1}) - [\mathbb{P}_h \tilde{V}_{h+1}](s^i_h, a^i_h))\right| \leq \sqrt{H},
\]
\[
\mathbb{E}[\eta_i^2 \mid \mathcal{F}_i] = \mathbb{E} \left[ 1\{[\mathbb{V}_h \tilde{V}](s^i_h, a^i_h) \leq \Tilde{\overline{\sigma}}_i^2 / (d^3 H)\} \cdot \Tilde{\overline{\sigma}}_i^{-2} [\mathbb{V}_h \tilde{V}](s^i_h, a^i_h) \right] \leq \frac{1}{d^3 H},
\]
\[
\max_i \left\{|\eta_i| \cdot \min\{1, \|x_i\|_{\Sigma^{-1}_{i,h}}\}\right\} \leq 2H \Tilde{\overline{\sigma}}_i^{-1} \|x_i\|_{\Tilde{\Lambda}^{-1}_{i,h}} \leq \frac{1}{d^3 H}.
\]
After taking a union bound over the $\varepsilon$-net, with probability at least $1 - \delta$, we have
\[
\left\|\sum_{i=1}^{k-1} x_i \eta_i \right\|_{\Tilde{\Lambda}^{-1}_{k,h}} \leq O\left(\sqrt{d \log^2\left(1 + \left(\frac{HK^{4}L^{2}d}{\delta \Tilde{\lambda}_{\Lambda}} \right) \right)} \right).
\]

In addition, on the event $\tilde{E}_{h+1}$ and $E$, according to Lemmas B.2 and B.3, we have
\[
\Tilde{\overline{\sigma}}^2_{i,h} \geq [\overline{\mathbb{V}}_{i,h} \Tilde{\widehat{V}}_{i,h+1}](s^k_h, a^k_h) + E_{i,h} + D_{i,h} + H \geq D_{i,h} + H \geq d^3 H [\mathbb{V}_h(\Delta \Tilde{\widehat{V}}_{k,h+1})](s^i_h, a^i_h) + H \geq d^3 H [\mathbb{V}_h \tilde{V}](s^i_h, a^i_h).
\]
Denote $\bar{V} = \Delta \Tilde{\widehat{V}}_{k,h+1} - \tilde{V}$. Then, 
\[
\begin{aligned}
    &\norm{ 
    \sum_{i=1}^{k-1} \Tilde{\Bar{\sigma}}_{i, h}^{-2} \phi \left( s_{h}^{i}, a_{h}^{i} \right) \left( \Delta \Tilde{\widehat{V}}_{k, h+1}(s_{h+1}^{i}) - \mathbb{P}_{h}\Delta \Tilde{\widehat{V}}_{k, h+1}(s_{h}^{i}, a_{h}^{i}) \right)}_{\Tilde{\Lambda}_{h, k}^{-1}} \\
    &\leq 2 \left\|\sum_{i=1}^{k-1} \Tilde{\overline{\sigma}}_i^{-2} \phi(s^i_h, a^i_h)(\tilde{V}(s^i_{h+1}) - [\mathbb{P}_h \tilde{V}](s^i_h, a^i_h))\right\|_{\Tilde{\Lambda}^{-1}_{k,h}} \\
    &\quad + 2 \left\|\sum_{i=1}^{k-1} \Tilde{\overline{\sigma}}_i^{-2} \phi(s^i_h, a^i_h)(\bar{V}(s^i_{h+1}) - [\mathbb{P}_h \bar{V}](s^i_h, a^i_h))\right\|_{\Tilde{\Lambda}^{-1}_{k,h}} \\
    &\leq 2 \left\|\sum_{i=1}^{k-1} x_i \eta_i \right\|_{\Tilde{\Lambda}^{-1}_{k,h}} + \frac{8\varepsilon^2 k^2}{\lambda} \\
    &\leq O\left(\sqrt{d \log^2\left(1 + \left(\frac{HK^{4}L^{2}d}{\delta \Tilde{\lambda}_{\Lambda}} \right) \right)} \right).
\end{aligned}
\]
where the first inequality holds from $\norm{a + b}_{\Sigma} \leq 2\norm{a}_{\Sigma} + 2\norm{b}_{\Sigma}$, the second inequality holds from $| \overline{V}(s^{\prime})| \leq \varepsilon$, $\norm{\phi(s, a)}_{2} \leq 1$, and $\Tilde{\Lambda}_{k, h} \succeq \Tilde{\lambda}_{\Lambda}$, $\Tilde{\Lambda}_{k, h}^{-1} \preceq 1$ and the last inequality holds with $\varepsilon = \sqrt{\Tilde{\lambda}_{\Lambda}} / (4H^{2}d^{2}K)$. Combining these results, we get
\[
\left| \Tilde{\widehat{w}}_{k,h^{\prime}}^\top \phi(s, a) - [\mathbb{P}_h \Tilde{\widehat{V}}_{k,h^{\prime}+1}](s, a) \right| \leq \beta \sqrt{\phi(s, a)^\top \Tilde{\Lambda}_{k,h^{\prime}}^{-1} \phi(s, a)}
\]
where 
\[
\beta = O\left(HL\sqrt{d\Tilde{\lambda}_{\Lambda}} + \sqrt{d \log^2\left(1 + \left(\frac{HK^{4}L^{2}d}{\delta \Tilde{\lambda}_{\Lambda}} \right) \right)} \right)
\]

\end{proof}

\section{Estimated Variance and Regret Bound Proofs}
\label{appendix:estimated-variance-regret-bound-section}

We simply state some results derived by \citetseparate{he2023nearlyminimaxoptimalreinforcement}. Our results are largely the same except for factors like $\iota$ and unlike \citetseparate{he2023nearlyminimaxoptimalreinforcement}, we must retain these terms since $\Tilde{\lambda}_{\Lambda}$ has an upper-bound that is induced from noise added to the privatized estimators and is not a regular constant like $\lambda$ in regular ridge-regression

\begin{lemma}[Lemma 4.4 From \citetseparate{zhou2022computationallyefficienthorizonfreereinforcement}]
    For any parameters $\beta' \geq 1$ and $C \geq 1$, the summation of bonuses is upper bounded by
\[
\sum_{k=1}^K \min \left( \beta' \sqrt{\phi(s_{k}^{h}, a_{k}^{h})^\top \Tilde{\Lambda}_{k,h}^{-1} \phi(s_{k}^{h}, a_{k}^{h})}, C \right)
\leq 4d^4H^6C\iota + 10\beta'd^5H^4\iota + 2\beta' \sqrt{2d\iota \sum_{k=1}^K (\Tilde{\sigma}^2_{k,h} + H)},
\]
where $\iota = \log \left(1 + \frac{K}{d\Tilde{\lambda}_{\Lambda}}\right)$.
\end{lemma}

\begin{lemma}[Lemma C.1 From \citetseparate{he2023nearlyminimaxoptimalreinforcement}]
    Define $\mathcal{E}_{1}$ as the following event
    \begin{align*}
    \mathcal{E}_{1} = \Bigl\{ 
    &\forall h \in [H], \; 
    \sum_{k=1}^{K} \sum_{h' = h}^{H} \bigl[\mathbb{P}_h\bigl(\widehat{V}_{k,h+1} - \widehat{V}^{\widehat{\pi}_k}_{k,h+1}\bigr)\bigr](s^k_h, a^k_h) \\
    &\quad - \sum_{k=1}^{K} \sum_{h' = h}^{H} \bigl(\widehat{V}_{k,h+1}(s^k_{h+1}) - \widehat{V}^{\widehat{\pi}_k}_{k,h+1}(s^k_{h+1})\bigr) \leq 2\sqrt{2H^3K \log\left(H / \delta \right)}
    \Bigr\}.
    \end{align*}
    Then, $\mathrm{Pr} \left( \mathcal{E}_{1} \right) \geq 1 - \delta$. Furthermore, on the events $\tilde{\mathcal{E}}$, $\mathcal{E}$, and $\mathcal{E}_1$, for all stages $h \in [H]$, the regret in the first $K$ episodes is upper bounded by:

\[
\sum_{k=1}^K \left(\Tilde{\widehat{V}}_{k,h}(s_h^k) - \Tilde{\widehat{V}}^{\Tilde{\pi}^{k}}_{k,h}(s_h^k)\right) 
\leq 16d^4H^8\iota + 40\beta d^7H^5\iota + 8\beta\sqrt{2dH\iota}
\sum_{h=1}^H \sum_{k=1}^K (\Tilde{\sigma}_{k,h}^2 + H) + 4\sqrt{H^3K \log(H/\delta)},
\]
and for all stages $h \in [H]$, we further have:
\begin{align*}
    \sum_{k=1}^K \sum_{h=1}^H 
\mathbb{P}_h \left(\Tilde{\widehat{V}}_{k,h}(s_h^k) - \Tilde{\widehat{V}}^{\Tilde{\pi}^{k}}_{k,h}(s_h^k)\right)(s_h^k, a_h^k)
\leq 16d^4H^9\iota + 40\beta d^7H^6\iota &+ 8H\beta\sqrt{2dH\iota}
\sum_{h=1}^H \sum_{k=1}^K (\Tilde{\sigma}_{k,h}^2 + H) \\
&+ 4\sqrt{H^5K \log(H/\delta)},
\end{align*}

where $\iota = \log \left(1 + \frac{K}{d\Tilde{\lambda}_{\Lambda}}\right)$.
\end{lemma}

\begin{lemma}Lemma C.2 From \citetseparate{he2023nearlyminimaxoptimalreinforcement}
    \label{lem:diff-between-opt-pess}
    Define $\mathcal{E}_{2}$ as the following event
    \begin{align*}
    \mathcal{E}_{2} = \Bigl\{ 
    &\forall h \in [H], \; 
    \sum_{k=1}^{K} \sum_{h' = h}^{H} \bigl[\mathbb{P}_h\bigl(\widehat{V}_{k,h+1} - \widecheck{V}_{k,h+1}\bigr)\bigr](s^k_h, a^k_h) \\
    &\quad - \sum_{k=1}^{K} \sum_{h' = h}^{H} \bigl(\widehat{V}_{k,h+1}(s^k_{h+1}) - \widecheck{V}_{k,h+1}(s^k_{h+1})\bigr) \leq 2\sqrt{2H^3K \log\left(H / \delta \right)}
    \Bigr\}.
    \end{align*}
    Then, $\mathrm{Pr} \left( \mathcal{E}_{2} \right) \geq 1 - \delta$. On the events $\tilde{\mathcal{E}}$, $\mathcal{E}$, and $\mathcal{E}_2$, the difference between the optimistic value function $\tilde{\widehat{V}}_{k,h}$ and the pessimistic value function $\tilde{\widecheck{V}}_{k,h}$ is upper bounded by:
    \begin{align*}
        \sum_{k=1}^K \sum_{h=1}^H 
\mathbb{P}_h \left(\Tilde{\widehat{V}}_{k,h}(s_h^k) - \tilde{\widecheck{V}}_{k,h+1}\right)(s_h^k, a_h^k)
\leq 32d^4H^9\iota + 40(\beta + \widehat{\beta})d^7H^6\iota &+ 8H(\beta + \widehat{\beta})\sqrt{2dH\iota}
\sum_{h=1}^H \sum_{k=1}^K (\Tilde{\sigma}_{k,h}^2 + H) \\
&+ 4\sqrt{H^5K \log(H/\delta)},
    \end{align*}

where $\iota = \log \left(1 + \frac{K}{d\Tilde{\lambda}_{\Lambda}}\right)$.
\end{lemma}

\begin{lemma}[Lemma C.3 From \citetseparate{he2023nearlyminimaxoptimalreinforcement}]
    \label{lem:est-variance-bound}
    On the event $\mathcal{E} \cap \tilde{\mathcal{E}} \cap \mathcal{E}_1 \cap \mathcal{E}_2 \cap \mathcal{E}_3$, the total estimated variance is upper bounded by:

\[
\sum_{k=1}^K \sum_{h=1}^H \Tilde{\sigma}_{k,h}^2 
\leq O\left(H^2K + H^{4.5}d^{3}K^{0.5}L^{2}\log^{1.5} \left( \frac{HK^{4}L^{2}d}{\delta \Tilde{\lambda}_{\Lambda}} \right)\right).
\]
\end{lemma}

\begin{lemma}
    \label{lemma:regret-bound}
    For any linear MDP $\mathcal{M}$, if we set the confidence radii $\widehat{\beta}$, $\widecheck{\beta}$, $\overline{\beta}$ as follows:

\[
    \widehat{\beta} = \widecheck{\beta} =  O\left(HL\sqrt{d \Tilde{\lambda}_{\Lambda}} + \sqrt{d^3 H^2 \log^2 \left(\frac{HK^{4}L^{2}d}{\delta \Tilde{\lambda}_{\Lambda}} \right)} \right),
\]
\[
\bar{\beta} = O\left(H^{2}L^{2} \sqrt{d\Tilde{\lambda}_{\Lambda}} + \sqrt{d^3 H^4 \log^2 \left(\frac{HK^{4}L^{2}d}{\delta \Tilde{\lambda}_{\Lambda}} \right)} \right).
\]
\[
    \beta = O\left(HL\sqrt{d\Tilde{\lambda}_{\Lambda}} + \sqrt{d \log^2\left(1 + \left(\frac{HK^{4}L^{2}d}{\delta \Tilde{\lambda}_{\Lambda}} \right) \right)} \right)
\]
then with high probability of at least $1 - 7\delta$, the regret of DP-LSVI-UCB\textsuperscript{++} is upper bounded as follows:

\[
\text{Regret}(K) \leq \Tilde{O} \left( d\sqrt{H^3K} + \frac{H^{15/4}d^{7/6}K^{1/2} \log(10dKH / \delta)}{\epsilon} \right)
\]

In addition, the number of updates for $\Tilde{\widehat{Q}}_{k,h}$ and $\Tilde{\widecheck{Q}}_{k,h}$ is upper bounded by $O(dH \log(1 + K/d\Tilde{\lambda}_{\Lambda}))$.
\end{lemma}

\begin{proof}[Proof of Lemma~\ref{lemma:regret-bound}]
    On the event $\mathcal{E} \cap \tilde{\mathcal{E}} \cap \mathcal{E}_1 \cap \mathcal{E}_2 \cap \mathcal{E}_3$, the regret is upper bounded by:
    \begin{align*}
        \text{Regret}(K) &= 
\sum_{k=1}^K \left(V^*_1(s_1^k) - \Tilde{\widehat{V}}^{\pi_k}_{k,1}(s_1^k)\right) \\
&\leq 
\sum_{k=1}^K \left(\Tilde{\widehat{V}}_{k,1}(s_1^k) - \Tilde{\widehat{V}}^{\widehat{\pi}^{k}}_{k,1}(s_1^k)\right) \\
&\leq 16d^4H^8\iota + 40\beta d^7H^5\iota + 8\beta\sqrt{2dH\iota
\sum_{h=1}^H \sum_{k=1}^K (\tilde{\sigma}_{k,h}^2 + H)} + 4\sqrt{H^3K \log(H/\delta)} \\
&\leq \Tilde{O} \left( d\sqrt{H^3K} + \frac{H^{18/4}d^{7/6}K^{1/2} \log(10dKH / \delta)}{\epsilon} \right)
    \end{align*}

where $\iota = \log(1 + K/(d\Tilde{\lambda}_{\Lambda}))$. The first inequality holds due to optimism (Lemma~\ref{lemma:optimism-and-pessimism}), the second inequality holds due to Lemma~\ref{lem:diff-between-opt-pess}, and the last inequality holds due to the variance bound (Lemma~\ref{lem:est-variance-bound}). Since the event $\mathcal{E} \cap \tilde{\mathcal{E}} \cap \mathcal{E}_1 \cap \mathcal{E}_2 \cap \mathcal{E}_3$ holds with probability at least $1 - 7\delta$ holds. In addition, according to Lemma~\ref{lem:dp-lsvi-ucb-switching-cost}, the number of updates for $\Tilde{\widehat{Q}}_{k,h}$ and $\Tilde{\widecheck{Q}}_{k,h}$ is upper bounded by $O(dH \log(1 + K/\Tilde{\lambda}_{\Lambda}))$. 
\end{proof}

\section{Switching Cost Proof}
\label{appendix:switching-cost}
We first prove a standard determinant upper bound for our privatized Gram matrix $\Tilde{\Lambda}$. This will be useful for determining the switching cost

\begin{lemma}[Privatized Determinant Upper Bound (Similar to Lemma C.1 in \citetseparate{DBLP:journals/corr/abs-2101-02195}]
\label{lemma:det-upper-bound}
Let $\left\{ \Tilde{\Lambda}_{h, k}, (h, k) \in [H] \times [K] \right\}$ be defined as in Algorithm~\ref{alg:dp-lsvi_ucb++}. Then, for all $h \in [H], k \in [K]$, we have that $\mathrm{det} \left( \Tilde{\Lambda}_{h, k} \right) \leq (\Tilde{\lambda}_{\Lambda}  + (k-1)/d)^{d}$.
\end{lemma}

\begin{proof}[Proof of Lemma~\ref{lemma:det-upper-bound}]
We have that 

\begin{align*}
    \mathrm{Tr} \left( \Tilde{\Lambda}_{h, k} \right) &= \mathrm{Tr}(K_{1}) + \sum_{i=1}^{k-1} \bar{\sigma}_{i, h}^{-2} \phi(s_{h}^{i}, a_{h}^{i}) \phi(s_{h}^{i}, a_{h}^{i})^{\top} \\
    &\leq d\Tilde{\lambda}_{\Lambda}  + \sum_{i=1}^{k-1} \bar{\sigma}_{i, h}^{-2} \phi(s_{h}^{i}, a_{h}^{i}) \phi(s_{h}^{i}, a_{h}^{i})^{\top} \\
    &\leq d\Tilde{\lambda}_{\Lambda}  + \sum_{i=1}^{k-1} ||\phi(s_{h}^{i}, a_{h}^{i})||_{2} \\
    &\leq d\Tilde{\lambda}_{\Lambda}  + k-1
\end{align*}

where the first inequality holds from the fact that for a symmetric matrix $A$, we have the inequality $\mathrm{Tr} \left( A \right) \leq n ||A||_{2}$ and that from the utility analysis (Lemma~\ref{lemma:private-utility-analysis}), $||K_{1}||_{2} \leq \Tilde{\lambda}_{\Lambda} $. The second inequality holds from the fact that $\bar{\sigma}_{i, h}^{-2} \leq 1$, and the last inequality holds from the assumption that $||\phi(s_{h}^{i}, a_{h}^{i})||_{2} \leq 1$. Now, since we have that $\Tilde{\Lambda}_{h, k}$ is positive semi-definite, by the AM-GM inequality, we have 
\[
    \mathrm{det} \left( \Tilde{\Lambda}_{h, k} \right) \leq \left( \frac{\mathrm{Tr} \left( \Tilde{\Lambda}_{h, k} \right)}{d} \right)^{d} \leq \left( \Tilde{\lambda}_{\Lambda}  + \frac{k-1}{d}  \right)^{d}
\]
\end{proof}
We can finally prove the switching cost of Algorithm~\ref{alg:dp-lsvi_ucb++}.
\begin{lemma}
    \label{lem:dp-lsvi-ucb-switching-cost}
    Conditioned on the event that $||K_{1}||_{2} \leq \Tilde{\lambda}_{\Lambda}$ for all $h, k \in [H] \times [K]$, DP-LSVI-UCB\textsuperscript{++} (Algorithm~\ref{alg:dp-lsvi_ucb++}) has a global switching cost of atmost $O \left(dH \log(1 + K/d\Tilde{\lambda}_{\Lambda} ) \right)$.
\end{lemma}

\begin{proof}
    We denote $k_{0} = 0$ and suppose that $\left\{ k_{1}, \dots, k_{m} \right\}$ be the episodes where our algorithm updates the value function. Then, according to the determinant-based criterion (Line~\ref{alg:det-condition}), for each episode $k_{i}$, there exists an $h \in [H]$ such that $\det(\Tilde{\Lambda}_{k,h}) \geq 2 \det(\Tilde{\Lambda}_{k_{i-1},h})$. Then, due to the utility analysis (Lemma~\ref{lemma:private-utility-analysis}), for $h^{\prime} \neq h$, we have $\Tilde{\Lambda}_{k_{i},h^{\prime}} \succeq \Tilde{\Lambda}_{k_{i-1},h^{\prime}}$. Thus,
    \[
        \prod_{h=1}^{H} \det(\Tilde{\Lambda}_{k,h}) \geq 2 \prod_{h=1}^{H}\det(\Tilde{\Lambda}_{k_{i-1},h})
    \]
    Applying this across all episodes, we get 
    \[
    \prod_{h=1}^{H} \det(\Tilde{\Lambda}_{k,h}) \geq 2^{m} \prod_{h=1}^{H}\det(\Tilde{\Lambda}_{k_{0},h}) = 2^{m} \prod_{h=1}^{H}\det( 2\Tilde{\lambda}_{\Lambda} I) = 2^{m} \Tilde{\lambda}_{\Lambda}^{dH}
    \]
    Furthermore, from Lemma~\ref{lemma:det-upper-bound}
    \[
        \prod_{h=1}^{H} \det(\Tilde{\Lambda}_{k,h}) \leq \left( \Tilde{\lambda}_{\Lambda} + \frac{K}{d}  \right)^{dH}
    \]
    Combining these two inequalities, we conclude with
    \[
        m \leq O \left( dH \log \left( 1 + K/d\Tilde{\lambda}_{\Lambda} \right)   \right)
    \]
\end{proof}

\section{Weight Norm Proofs}
We prove upper bounds on the optimistic, pessimistic, and squared weight vectors. These will be used in uniform covering arguments which are used in our regret analysis.

\begin{lemma}
    \label{lemma:optimistic-weight-norm}
    For all stages $h \in [H]$ and all episodes $n \in \mathbb{N}$, the norm of the weight vector $\Tilde{\widehat{w}}_{k, h}$ can be upper bounded as 
    \[
        \norm{\Tilde{\widehat{w}}_{k, h}}_{2} \leq HKL \sqrt{\frac{2d}{\Tilde{\lambda}_{\Lambda}}}
    \]
\end{lemma}

\begin{proof}[Proof of Lemma~\ref{lemma:optimistic-weight-norm}]
    First, recall that by definition we have 
    \[
        \Tilde{\widehat{w}}_{k, h} = \Tilde{\Lambda}_{k, h}^{-1} \left[ \sum_{i=1}^{k-1} \Tilde{\Bar{\sigma}}_{i, h}^{-2} \phi \left( s_{h}^{i}, a_{h}^{i} \right)\Tilde{\widehat{V}}_{k, h+1}(s_{h+1}^{i}) + \phi_{1} \right]
    \]
    Then, we have 
    \begin{align*}
        \norm{\Tilde{\widehat{w}}}_{2}^{2} &= \norm{\Tilde{\Lambda}_{k, h}^{-1} \left[ \sum_{i=1}^{k-1} \Tilde{\Bar{\sigma}}_{i, h}^{-2} \phi \left( s_{h}^{i}, a_{h}^{i} \right)\Tilde{\widehat{V}}_{k, h+1}(s_{h+1}^{i}) + \phi_{1} \right]}_{2}^{2} \\
        &\leq k \sum_{i=1}^{k-1} \norm{\Tilde{\Lambda}_{k, h}^{-1}\Tilde{\Bar{\sigma}}_{i, h}^{-2} \phi \left( s_{h}^{i}, a_{h}^{i} \right)\Tilde{\widehat{V}}_{k, h+1}(s_{h+1}^{i}) + \Tilde{\Lambda}_{k, h}^{-1} \phi_{1}}_{2}^{2} \\
        &\leq k \sum_{i=1}^{k-1} \norm{\Tilde{\Lambda}_{k, h}^{-1}\Tilde{\Bar{\sigma}}_{i, h}^{-2} \phi \left( s_{h}^{i}, a_{h}^{i} \right)\Tilde{\widehat{V}}_{k, h+1}(s_{h+1}^{i})}_{2}^{2} + k \sum_{i=1}^{k-1} \norm{\Tilde{\Lambda}_{k, h}^{-1} \phi_{1}}_{2}^{2} \\
        &\leq \frac{kH^{2}}{\Tilde{\lambda}_{\Lambda}} \sum_{i=1}^{k-1} \Tilde{\Bar{\sigma}}_{i, h}^{-2} \phi \left( s_{h}^{i}, a_{h}^{i} \right)^{\top} \Tilde{\Lambda}_{k, h}^{-1} \phi \left( s_{h}^{i}, a_{h}^{i} \right) + \frac{k^{2}L^{2}}{\Tilde{\lambda}_{\Lambda}^{2}} \\
        &\leq \frac{kH^{2}}{\Tilde{\lambda}_{\Lambda}} \mathrm{trace} \left( \Tilde{\Lambda}_{k, h}^{-1} \sum_{i=1}^{k-1} \Tilde{\Bar{\sigma}}_{i, h}^{-2} \phi \left( s_{h}^{i}, a_{h}^{i} \right)^{\top}  \phi \left( s_{h}^{i}, a_{h}^{i} \right) \right) + \frac{k^{2}L^{2}}{\Tilde{\lambda}_{\Lambda}^{2}}
    \end{align*}
    where the first inequality holds from Cauchy-Schwartz, the second inequality holds from triangle inequality, and the third inequality holds from the fact that $\Tilde{\widehat{V}}_{k, h+1} \leq H$, $\norm{\Tilde{\Lambda}_{k, h}^{-1}}_{2} \leq \frac{1}{\Tilde{\lambda}_{\Lambda}}$, and $\norm{\phi_{1}}_{2} \leq L$ from the utility analysis (Lemma ~\ref{lemma:private-utility-analysis}). Now, we assume the eigen-decomposition of matrix $\sum_{i=1}^{k-1} \Tilde{\Bar{\sigma}}_{i, h}^{-2} \phi \left( s_{h}^{i}, a_{h}^{i} \right)^{\top}\phi \left( s_{h}^{i}, a_{h}^{i} \right)$ is $Q^{\top} \Sigma Q$. Then, we have
    \begin{align*}
        \mathrm{trace} \left( \Tilde{\Lambda}_{k, h}^{-1} \sum_{i=1}^{k-1} \Tilde{\Bar{\sigma}}_{i, h}^{-2} \phi \left( s_{h}^{i}, a_{h}^{i} \right)^{\top} \phi \left( s_{h}^{i}, a_{h}^{i} \right) \right) &= \mathrm{trace} \left( \left(  Q^{\top} \Sigma Q + 2\Tilde{\lambda}_{\Lambda}I_{d}\right)^{-1} Q^{\top} \Sigma Q \right) \\
        &= \mathrm{trace} \left(  \left( \Sigma + 2\Tilde{\lambda}_{\Lambda}I_{d} \right)^{-1} \Sigma   \right) \\
        &= \sum_{i=1}^{d} \frac{\sigma_{i}}{\sigma_{i} + 2\Tilde{\lambda}_{\Lambda}I_{d}} \\
        &\leq d
    \end{align*}
Thus, putting these together, we get 
\[
    \norm{\Tilde{\widehat{w}}_{k, h}}_{2}^{2} \leq \frac{kH^{2}d}{\Tilde{\lambda}_{\Lambda}} + \frac{k^{2}L^{2}}{\Tilde{\lambda}_{\Lambda}^{2}} \leq \frac{2k^{2}H^{2}L^{2}d}{\Tilde{\lambda}_{\Lambda}}
\]
\end{proof}

The same analysis holds for the pessimistic weight vector $\Tilde{\widecheck{w}}$

\begin{lemma}
    \label{lemma:pessimistic-weight-norm}
    For all stages $h \in [H]$ and all episodes $n \in \mathbb{N}$, the norm of the weight vector $\Tilde{\widecheck{w}}_{k, h}$ can be upper bounded as 
    \[
        \norm{\Tilde{\widecheck{w}}_{k, h}}_{2} \leq HKL \sqrt{\frac{2d}{\Tilde{\lambda}_{\Lambda}}}
    \]
\end{lemma}

\begin{proof}[Proof of Lemma~\ref{lemma:pessimistic-weight-norm}]
    The proof is exactly the same as Lemma~\ref{lemma:optimistic-weight-norm} except we use the pessimistic value function class $\widecheck{\mathcal{V}}_{h}$.
\end{proof}

Likewise, using similar analysis as above, we can also bound the weight vector $\Tilde{\overline{w}}_{k, h}$.

\begin{lemma}
    \label{lemma:other-weight-norm}
    For all stages $h \in [H]$ and all episodes $n \in \mathbb{N}$, the norm of the weight vector $\Tilde{\overline{w}}_{k, h}$ can be upper bounded as 
    \[
        \norm{\Tilde{\overline{w}}_{k, h}}_{2} \leq H^{2}KL \sqrt{\frac{2d}{\Tilde{\lambda}_{\Lambda}}}
    \]
\end{lemma}

\begin{proof}[Proof of Lemma~\ref{lemma:other-weight-norm}]
    The proof is exactly the same as Lemma~\ref{lemma:optimistic-weight-norm} except we use the pessimistic value function class $\widehat{\mathcal{V}}_{h}^{2}$.
\end{proof}

\section{Covering Argument Results}

\begin{lemma}[Lemma D.5 from \citetseparate{jin2020provably}]
\label{lemma:jin-euclidean-ball}
For a Euclidean ball with radius \( R \) in \( \mathbb{R}^d \), the \( \varepsilon \)-covering number of this ball is upper bounded by
\[
(1 + 2R / \varepsilon)^d.
\]
\end{lemma}

With the help of Lemma F.5, the covering number \( \mathcal{N}_\varepsilon \) of optimistic function class \( \widehat{\mathcal{V}}_{h} \) can be upper bounded by the following lemma:

\begin{lemma}[Lemma F.6 from \citetseparate{he2023nearlyminimaxoptimalreinforcement}]
\label{lemma:optimistic-value-function-covering}
For optimistic function class \( \widehat{\mathcal{V}}_h \), 
\[
    \widehat{\mathcal{V}}_h = \Bigg\{ V \;\Bigg|\; V(\cdot) = \max_a \min_{1 \leq i \leq l} \min \Bigg( 
    H, r_h(\cdot, a) + w_i^\top \phi(\cdot, a) \\
    + \beta \sqrt{\phi(\cdot, a)^\top \Tilde{\Lambda}_i^{-1} \phi(\cdot, a)} \Bigg), 
    \|w_i\| \leq L_1, \Tilde{\Lambda}_i \succeq \Tilde{\lambda}_{\Lambda} I \Bigg\}
\]
where $l = dH \log \left( 1 + K / d \Tilde{\lambda}_{\Lambda} \right)$ and $L_1 = HKL \sqrt{\frac{2d}{\Tilde{\lambda}_{\Lambda}}}$. Define the distance between two functions \( V_1 \) and \( V_2 \) as \( V_1, V_2 \in \widehat{\mathcal{V}}_h \) as $\text{dist}(V_1, V_2) = \max_s |V_1(s) - V_2(s)|$. With respect to this distance function, the \( \varepsilon \)-covering number \( \mathcal{N}_\varepsilon \) of the function class \( \mathcal{V}_h \) can be upper bounded by
\[
\log \mathcal{N}_\varepsilon \leq dl \log(1 + 4L_{1} / \varepsilon) + d^2l \log \left( 1 + 8\sqrt{d} \beta^2 / \left( \Tilde{\lambda}_{\Lambda} \varepsilon^{2} \right) \right).
\]
\end{lemma}

\begin{proof}[Proof of Lemma~\ref{lemma:optimistic-value-function-covering}]
By letting $\Sigma = \beta^{2} \left( \Tilde{\Lambda} \right)^{-1}$, we can reparametrize the function class \( \widehat{\mathcal{V}}_h \) as 

\[
    \widehat{\mathcal{V}}_h = \Bigg\{ V \;\Bigg|\; V(\cdot) = \max_a \min_{1 \leq i \leq l} \min \Bigg( 
    H, r_h(\cdot, a) + w_i^\top \phi(\cdot, a) \\
    + \sqrt{\phi(\cdot, a)^\top \Sigma \phi(\cdot, a)} \Bigg), 
    \|w_i\| \leq L_1, \Sigma \succeq \beta^{2} \Tilde{\lambda}_{\Lambda} I \Bigg\}
\]
For any two functions $V_{1}, V_{2} \in \widehat{\mathcal{V}}_h$, let them take the form as seen above. Then, since $\mathrm{min} \left\{H, \cdot \right\}$, $\min_{1 \leq i \leq l}$, and $\max_a$ are contraction maps, we have 

\begin{align*}
    \mathrm{dist} \left( V_{1}, V_{2} \right) &= \max_{s \in \mathcal{S}} \left| V_{1}(s) - V_{2}(s) \right| \\
    &\leq \max_{1 \leq i \leq l, s \in \mathcal{S}, a \in \mathcal{A}} \left| w_{1, i}^\top \phi(\cdot, a)
    + \sqrt{\phi(\cdot, a)^\top \Sigma_{1, i} \phi(\cdot, a)} \Bigg) -  w_{2, i}^\top \phi(\cdot, a)
    + \sqrt{\phi(\cdot, a)^\top \Sigma_{2, i} \phi(\cdot, a)} \Bigg)\right| \\
    &\leq \max_{1 \leq i \leq l, s \in \mathcal{S}, a \in \mathcal{A}} \left| \left(w_{1, i} - w_{2, i}\right)\phi(s, a)  \right|  + \max_{1 \leq i \leq l, s \in \mathcal{S}, a \in \mathcal{A}} \left| \sqrt{\phi(s, a)^{\top} \left( \Sigma_{1, i} - \Sigma_{2, i} \right) \phi(s, a)} \right| \\
    &\leq \max_{1 \leq i \leq l} \norm{w_{1, i} - w_{2, i}}_{2} + \max_{1 \leq i \leq l} \sqrt{\norm{\Sigma_{1, i} - \Sigma_{2, i}}_{F}}
\end{align*}
where the first inequality holds due to the contraction property, the second inequality holds due to the fact that $\max_{x} \left| f(x) + g(x) \right| \leq \max_{x} \left| f(x) \right| + \max_{x} \left| g(x) \right|$ and $\left| \sqrt{x} - \sqrt{y} \right| \leq \sqrt{\left| x - y \right|}$, and the last inequality holds from $\norm{\phi(s, a}_{2} \leq 1$. Now, let $\mathcal{C}_{w}$ be a $\varepsilon / 2$ covering net of $\left\{ w \in \mathbb{R}^{d} \mid \norm{w}_{2} \leq L_{1} \right\}$ and let $\mathcal{C}_{\Sigma}$ be a $\varepsilon^{2} / 4$ covering net of $\left\{ \Sigma \in \mathbb{R}^{d \times d} \mid \norm{\Sigma}_{F} \leq d^{1/2}\beta^{2}\Tilde{\lambda}_{\Lambda}^{-1} \right\}$. By Lemma~\ref{lemma:jin-euclidean-ball}, we know
\[
    \left| \mathcal{C}_{w}  \right| \leq \left( 1 + 4L / \varepsilon \right)^{d}, \; \left| \mathcal{C}_{\Sigma}  \right| \leq \left( 1 + 8d^{1/2}\beta^{2} / \left( \Tilde{\lambda}_{\Lambda} \varepsilon^{2} \right) \right)^{d^{2}}
\]
We know that for any $V_{1}, V_{2} \in \widehat{\mathcal{V}}_{h}$, there exists $w_{1}, w_{2} \in \mathcal{C}_{w}$ and $\Sigma_{1}, \Sigma_{2} \in \mathcal{C}_{\Sigma}$ such that $\mathrm{dist}\left( V_{1}, V_{2} \right) \leq \varepsilon$. Thus, this means that the covering number $\left| N_{\varepsilon} \right| \leq \left| \mathcal{C}_{w}  \right|^{l} \left| \mathcal{C}_{\Sigma}  \right|^{l}$. Thus, taking logs, we get 
\[
\log \mathcal{N}_\varepsilon \leq dl \log(1 + 4L_{1} / \varepsilon) + d^2l \log \left( 1 + 8\sqrt{d} \beta^2 / \left( \Tilde{\lambda}_{\Lambda} \varepsilon^{2} \right) \right).
\]
\end{proof}

Likewise, we can also upper bound the covering number of the pessimistic function class \( \widecheck{\mathcal{V}}_{h} \)

\begin{lemma}[Lemma F.7 from \citetseparate{he2023nearlyminimaxoptimalreinforcement}]
\label{lemma:pessimistic-value-function-covering}
For pessimistic function class \( \widecheck{\mathcal{V}}_h \),
\[
    \widecheck{\mathcal{V}}_h = \Bigg\{ V \;\Bigg|\; V(\cdot) = \max_a \max_{1 \leq i \leq l} \max \Bigg( 
    H, r_h(\cdot, a) + w_i^\top \phi(\cdot, a) \\
    - \beta \sqrt{\phi(\cdot, a)^\top \Tilde{\Lambda}_i^{-1} \phi(\cdot, a)} \Bigg), 
    \|w_i\| \leq L_1, \Tilde{\Lambda}_i \succeq \Tilde{\lambda}_{\lambda} I \Bigg\}
\]
where $l = dH \log \left( 1 + K / d \Tilde{\lambda}_{\Lambda} \right)$ and $L_1 = HKL \sqrt{\frac{2d}{\Tilde{\lambda}_{\Lambda}}}$. Define the distance between two functions \( V_1 \) and \( V_2 \) as \( V_1, V_2 \in \widecheck{\mathcal{V}}_h \) as $\text{dist}(V_1, V_2) = \max_s \lvert V_1(s) - V_2(s) \rvert$.
With respect to this distance function, the \( \varepsilon \)-covering number \( \mathcal{N}_\varepsilon \) of the function class \( \check{\mathcal{V}}_h \) can be upper bounded by
\[
\log \mathcal{N}_\varepsilon \leq dl \log(1 + 4L_{1} / \varepsilon) + d^2l \log \left( 1 + 8\sqrt{d} \beta^2 / \left( \Tilde{\lambda}_{\Lambda} \varepsilon^{2} \right) \right).
\]
\end{lemma}

Now that we have these results, the only result we require is an upper bound on the covering number of the optimistic value function class squared. This result is provided below

\begin{lemma}[Lemma F.7 from \citetseparate{he2023nearlyminimaxoptimalreinforcement}]
\label{lemma:optimistic-value-function-covering-squared}
For the squared function class $\widehat{\mathcal{V}}_h^2$, we define the distance between two functions $V_1^2$ and $V_2^2$ in $\widehat{\mathcal{V}}_h^2$ as:
\[
\text{dist}(V_1^2, V_2^2) = \max_s \big|V_1^2(s) - V_2^2(s)\big|.
\]
With respect to this distance function, the $\varepsilon$-covering number $N_\varepsilon$ of the function class $\widehat{\mathcal{V}}_h^2$ can be upper bounded by:
\[
\log N_\varepsilon \leq dl \log\left(1 + 8HL_{1} / \varepsilon \right) + d^{2}l \log\left(1 + 32\sqrt{d}H^2\beta^2 / \left( \Tilde{\lambda}_{\Lambda} \varepsilon^{2} \right) \right).
\]
where $l = dH \log \left( 1 + K / d \Tilde{\lambda}_{\Lambda} \right)$ and $L_1 = HKL \sqrt{\frac{2d}{\Tilde{\lambda}_{\Lambda}}}$.
\end{lemma}

\section{Auxiliary Results}
\label{appendix:aux-results}

\begin{lemma} [Lemma G.1 From \citetseparate{he2023nearlyminimaxoptimalreinforcement}]
    \label{lem:variance-linear-mdp}
    For any stage $h \in [H]$ in a linear MDP and any bounded function $V : S \to [0, B]$, there always exists a vector $w \in \mathbb{R}^d$ such that for all state-action pairs $(s, a) \in S \times A$, we have
\[
[\mathbb{P}_h V](s, a) = w^\top \phi(s, a),
\]
where $\|w\|_2 \leq B\sqrt{d}$.
\end{lemma}

\begin{proof}[Proof of Lemma~\ref{lem:variance-linear-mdp}]
By assumption of the linear MDP setting, we have

\begin{align*}
    [\mathbb{P}_h V](s, a) =
\int \mathbb{P}_h(s' | s, a)V(s') \, ds' &= \int \phi(s, a)^\top V(s') \, d\theta_h(s') \\
&= \phi(s, a)^\top \int V(s') \, d\theta_h(s') \\
&= \phi(s, a)^\top w,
\end{align*}

where we set $w = \int V(s') \, d\theta_h(s')$. Additionally, the norm of $w$ is upper bounded by $\left\| \int V(s') \, d\theta_h(s') \right\| \leq \max_{s'} V(s') \cdot \sqrt{d} = B\sqrt{d}$.
\end{proof}

\begin{lemma}[Azuma-Hoeffding Inequality, \citetseparate{10.5555/1137817}]
Let $\{x_i\}_{i=1}^n$ be a martingale difference sequence with respect to a filtration $\{\mathcal{G}_i\}$ satisfying $|x_i| \leq M$ for some constant $M$, $x_i$ is $\mathcal{G}_{i+1}$-measurable, and $\mathbb{E}[x_i|\mathcal{G}_i] = 0$. Then, for any $0 < \delta < 1$, with probability at least $1 - \delta$, we have:
\[
\sum_{i=1}^n x_i \leq M \sqrt{2n \log(1/\delta)}.
\]
\end{lemma}

\begin{lemma}[Lemma 11 in \citetseparate{10.5555/2986459.2986717}]
    Let $\{x_k\}_{k=1}^K$ be a sequence of vectors in $\mathbb{R}^d$, and let $\Sigma_0$ be a $d \times d$ positive definite matrix. Define $\Sigma_k = \Sigma_0 + \sum_{i=1}^k x_i x_i^\top$. Then, we have:
\[
\sum_{i=1}^k \min\{1, x_i^\top \Sigma_{i-1}^{-1} x_i\} \leq 2 \log\left(\frac{\det \Sigma_k}{\det \Sigma_0}\right).
\]
In addition, if $\|x_i\|_2 \leq L$ for all $i \in [K]$, then:
\[
\sum_{i=1}^k \min\{1, x_i^\top \Sigma_{i-1}^{-1} x_i\} 
\leq 2 \log\left(\frac{\det \Sigma_k}{\det \Sigma_0}\right) 
\leq 2 \left( d \log\left(\frac{\mathrm{trace}(\Sigma_0) + kL^2}{d}\right) - \log \det \Sigma_0 \right).
\]
\end{lemma}

\begin{lemma}[Lemma 12 in \citetseparate{10.5555/2986459.2986717}]
    Suppose $A, B \in \mathbb{R}^{d \times d}$ are two positive definite matrices satisfying $A \preceq B$. Then, for any $x \in \mathbb{R}^d$:
\[
\|x\|_A \leq \|x\|_B \cdot \sqrt{\frac{\det(A)}{\det(B)}}.
\]
\end{lemma}

\begin{lemma}[Theorem 1 in \citetseparate{10.5555/2986459.2986717}]
\label{lemma:self-normalized-bound}
Let $\{ \mathcal{F}_t \}_{t=0}^\infty$ be a filtration. Let $\{ \eta_t \}_{t=1}^\infty$ be a real-valued stochastic process such that $\eta_t$ is $\mathcal{F}_t$-measurable and $\eta_t$ is conditionally $R$-sub-Gaussian for some $R \geq 0$, i.e.,
\[
\forall \lambda \in \mathbb{R}, \quad \mathbb{E} \left[ e^{\lambda \eta_t} \mid \mathcal{F}_{t-1} \right] \leq \exp\left( \frac{\lambda^2 R^2}{2} \right).
\]
Let $\{ x_t \}_{t=1}^\infty$ be an $\mathbb{R}^d$-valued stochastic process such that $x_t$ is $\mathcal{F}_{t-1}$-measurable. Assume that $Z$ is a $d \times d$ positive definite matrix. For any $k \geq 0$, define
\[
Z_k = Z + \sum_{s=1}^t X_s X_s^\top
\]
Then, for any $\delta > 0$, with probability at least $1 - \delta$, for all $t \geq 0$,
\[
\| \sum_{i=1}^k x_{i} \eta_i \|_{Z_k^{-1}}^2 \leq 2 R^2 \log\left( \frac{\det(Z_k)^{1/2} \det(Z)^{-1/2}}{\delta} \right).
\]
\end{lemma}

\begin{lemma}[Confidence Ellipsoid, Theorem 2 in \citetseparate{10.5555/2986459.2986717}]
\label{lemma:abbassi-confidence-theorem-2}
Let $\{\mathcal{G}_k\}_{k \geq 1}$ be a filtration, and $\{x_k, \eta_k\}_{k \geq 1}$ be a stochastic process such that $x_k \in \mathbb{R}^d$ is $\mathcal{G}_k$-measurable and $\eta_k \in \mathbb{R}$ is $\mathcal{G}_{k+1}$-measurable. Let $L, \sigma, \Sigma, \varepsilon > 0$, and $\mu^* \in \mathbb{R}^d$. For $k \geq 1$, let $y_k = \langle \mu^*, x_k \rangle + \eta_k$, and suppose that $\eta_k, x_k$ satisfy:
\[
\mathbb{E}[\eta_k | \mathcal{G}_k] = 0, \quad |\eta_k| \leq R, \quad \|x_k\|_2 \leq L.
\]
Define $Z_k = 2\Tilde{\lambda}_{\Lambda} I + \sum_{i=1}^k x_i x_i^\top + K_{1}$, $b_k = \sum_{i=1}^k y_i x_i$, $\mu_k = Z_k^{-1} b_k$, and:
\[
\beta_k = R \sqrt{d \log\left(\frac{1 + kL^2/\Tilde{\lambda}_{\Lambda}}{\delta}\right)}.
\]
Then, for any $0 < \delta < 1$, with probability at least $1 - \delta$, we have:
\[
\forall k \geq 1, \quad \left\|\sum_{i=1}^k x_i \eta_i\right\|_{Z_k^{-1}} \leq \beta_k, \quad \|\mu_k - \mu^*\|_{Z_k} \leq \beta_k + \sqrt{\lambda}\|\mu^*\|_2.
\]
\end{lemma}

\begin{proof}[Proof of Lemma~\ref{lemma:abbassi-confidence-theorem-2}]
We will prove the following determinant-trace inequality. The result will then hold by applying Lemma~\ref{lemma:self-normalized-bound}
\begin{lemma}[Determinant-Trace Inequality]
\label{lemma:determinant-trace-inequality}
Suppose $x_1, x_2, \dots, x_K \in \mathbb{R}^d$ and for any $1 \leq k \leq K$, $\|x_k\|_2 \leq L$. Let $Z_k = 2\Tilde{\lambda}_{\Lambda} I + \sum_{k=1}^K x_k x_k^\top + K_{1}$ where $\norm{K_{1}}_{2} \leq \Tilde{\lambda}_{\Lambda}$. Then,
\[
\det(Z_k) \leq (3\Tilde{\lambda}_{\Lambda} + kL^2 / d)^d.
\]
\end{lemma}

\begin{proof}[Proof of Lemma~\ref{lemma:determinant-trace-inequality}]
Let $\alpha_{1}, \alpha_{2}, \dots, \alpha_{d}$ denote the eigenvalues of $Z_{k}$. Recall that from the utility analysis (Lemma ~\ref{lemma:private-utility-analysis}), by construction, $Z_k$ must be positive-definite. Then, notice that $\mathrm{det} \left( Z_{k} \right) = \prod_{i=1}^{d} \alpha_{k}$ and $\mathrm{trace}\left( Z_{k} \right) = \sum_{i=1}^{d} \alpha_{k}$. By the AM-GM inequality
\[
    \sqrt[d]{\alpha_{1}\dots\alpha_{d}} \leq \frac{1}{d} \sum_{i=1}^{d} \alpha_{i}
\]

Thus, we have that $\mathrm{det} \left( Z_{k} \right) \leq \left( \mathrm{trace}\left( Z_{k} \right) / d \right)^{d}$. Furthermore, notice that 

\begin{align*}
    \mathrm{trace}\left( Z_{k} \right) &= \mathrm{trace}\left( 2\Tilde{\lambda}_{\Lambda} I \right) + \mathrm{trace}\left( \sum_{k=1}^K x_s x_s^\top \right) + \mathrm{trace}\left( K_{1} \right) \\
    &\leq 3d \Tilde{\lambda}_{\Lambda} + KL^{2}
\end{align*}
where the inequality holds from the assumption that $\|x_k\|_2 \leq L$ and $\norm{K_{1}}_{2} \leq \Tilde{\lambda}_{\Lambda}$ from the utility analysis. Thus, putting these together, we get
the claim
\end{proof}

We now use the above result. From Lemma~\ref{lemma:self-normalized-bound}, we have that 
\[
\| \sum_{i=1}^k x_{i} \eta_i \|_{Z_k^{-1}}^2 \leq 2 R^2 \log\left( \frac{\det(Z_k)^{1/2} \det(Z)^{-1/2}}{\delta} \right).
\]

In our case, we have $Z = 2\Tilde{\lambda}_{\Lambda}I$. Utilizing our determinant upper bound and the fact that $\mathrm{det} \left( Z \right) = \left( 2\Tilde{\lambda}_{\Lambda} \right)^{d}$, we have 

\begin{align*}
    \log\left( \frac{\det(Z_k)^{1/2}}{\det(Z)^{1/2}} \right) &\leq \log\left( \frac{\left( 3\Tilde{\lambda}_{\Lambda} + kL^2 / d \right)^{d/2}}{\left( 2\Tilde{\lambda}_{\Lambda} \right)^{d/2}} \right) \\
    &\leq \frac{d}{2} \log\left( 1 + kL^{2}/\Tilde{\lambda}_{\Lambda} \right)
\end{align*}
where the first inequality comes from Lemma~\ref{lemma:determinant-trace-inequality} and the last inequality holds just by upper bounding the first constant term in the logarithm. Thus, we get the claim simply by taking square roots.
\end{proof}

\begin{lemma}[Lemma 4.4 in \citetseparate{zhou2022computationallyefficienthorizonfreereinforcement}]
    \label{aux:zhou-and-gu}
    Let $\{\sigma_k, \hat{\beta}_k\}_{k \geq 1}$ be a sequence of non-negative numbers, $\alpha, \gamma > 0$, $\{a_k\}_{k \geq 1} \subset \mathbb{R}^d$, and $\|a_k\|_2 \leq A$. Let $\{\bar{\sigma}_k\}_{k \geq 1}$ and $\{\hat{\Sigma}_k\}_{k \geq 1}$ be recursively defined as follows:
\[
\hat{\Sigma}_1 = 2 \Tilde{\lambda}_{\Lambda} I, \quad \forall k \geq 1, \, \bar{\sigma}_k = \max\{\sigma_k, \alpha, \gamma \|a_k\|_{\hat{\Sigma}_k^{-1}}^{1/2}\}, \quad \hat{\Sigma}_{k+1} = \hat{\Sigma}_k + a_k a_k^\top / \bar{\sigma}_k^2
\]
Let $\iota = \log\left(1 + \frac{KA^2}{d\Tilde{\lambda}_{\Lambda}\alpha^2}\right)$. Then, we have:
\[
\sum_{k=1}^K \min\{1, \|a_k\|_{\hat{\Sigma}_k^{-1}}\} \leq 2d\iota + 2\gamma^2d\iota + 2\sqrt{d\iota} \sqrt{\sum_{k=1}^K (\sigma_k^2 + \alpha^2)}.
\]
\end{lemma}

\begin{proof}[Proof of Lemma~\ref{aux:zhou-and-gu}]
We refer readers to Lemma 4.4 in \citetseparate{zhou2022computationallyefficienthorizonfreereinforcement} for further details. Our proofs are identical except for our usage of Lemma~\ref{lemma:abbassi-confidence-theorem-2} which is why our $\iota$ term is different.
\end{proof}

\begin{lemma}[Theorem 4.3 in \citetseparate{zhou2022computationallyefficienthorizonfreereinforcement}]
\label{aux:zhou-and-gu-theorem}
Let $\{ \mathcal{G}_k \}_{k=1}^\infty$ be a filtration, and let $\{(x_k, \eta_k)\}_{k \geq 1}$ be a stochastic
process such that $x_k \in \mathbb{R}^d$ is $\mathcal{G}_k$-measurable and $\eta_k \in \mathbb{R}$
is $\mathcal{G}_{k+1}$-measurable. Let $L, \sigma > 0$, and $\mu^* \in \mathbb{R}^d$. For $k \geq 1$,
define $y_k = \langle \mu^*, x_k \rangle + \eta_k$. Suppose that $\eta_k, x_k$ also satisfy
\[
\mathbb{E}[\eta_k \mid \mathcal{G}_k] = 0, \quad \mathbb{E}[\eta_k^2 \mid \mathcal{G}_k] \leq \sigma^2, \quad |\eta_k| \leq R, \quad \|x_k\|_2 \leq L.
\]
For $k \geq 1$, let
\[
Z_k = \lambda I + \sum_{i=1}^k x_i x_i^\top, \quad b_k = \sum_{i=1}^k y_i x_i, \quad \mu_k = Z_k^{-1} b_k,
\]
and 
\[
\beta_k = \widetilde{O}\bigl(\sigma \sqrt{d} + \max_{1 \leq i \leq k}|\eta_i| \min\{1, \|x_i\|_{Z_{i-1}^{-1}}\}\bigr).
\]
Then, for any $0 < \delta < 1$, with probability at least $1 - \delta$, for all $k \in [K]$, we have
\[
\left\|\sum_{i=1}^k x_i \eta_i \right\|_{Z_k^{-1}} \leq \beta_k, \quad \text{and} \quad \|\mu_k - \mu^*\|_{Z_k} \leq \beta_k + \sqrt{\lambda}\|\mu^*\|_2.
\]
\end{lemma}

\end{document}